\documentclass{article}

\usepackage{PRIMEarxiv}

\usepackage[utf8]{inputenc} 
\usepackage[T1]{fontenc}    
\usepackage{hyperref}       
\usepackage{url}            
\usepackage{booktabs}       
\usepackage{amsfonts}       
\usepackage{nicefrac}       
\usepackage{microtype}      
\usepackage{lipsum}
\usepackage{fancyhdr}       
\usepackage{graphicx}       
\usepackage{xspace}
\usepackage{xcolor}
\usepackage{dsfont}
\usepackage{amsmath}
\usepackage{amsthm}
\usepackage{algorithm}
\usepackage{algpseudocode}
\usepackage{hyperref}
\usepackage{natbib} 
\usepackage{graphicx}
\usepackage{subcaption} 
\usepackage{caption}
\usepackage{comment}

\algnewcommand\algorithmicforeach{\textbf{for each}}

\algdef{S}[FOR]{ForEach}[1]{\algorithmicforeach\ #1\ \algorithmicdo}
\newtheorem{theorem}{Theorem}[section]   
\newtheorem*{theorem*}{Theorem}
\newtheorem{definition}{Definition}[section]
\newtheorem{lemma}{Lemma}[section]
\newtheorem*{lemma*}{Lemma}
\newtheorem{corollary}{Corollary}[section]
\newtheorem{proposition}{Proposition}[section]
\newtheorem*{proposition*}{Proposition}
\newtheorem*{corollary*}{Corollary}
\newtheorem{assumption}{Assumption}[section]

\newcommand*{\MyDef}{\mathrm{def}}
\newcommand*{\eqdef}{\ensuremath{\mathop{\overset{\MyDef}{=}}}}
\newcommand{\indep}{\perp\!\!\!\!\perp}

\graphicspath{{media/}}     

\title{The Geometry of Invariant Learning: An Information-Theoretic Analysis of Data Augmentation and Generalization}

\author{
 Abdelali Bouyahia 
  \\
  Département d'informatique et génie logiciel\\
  Universit\'e Laval\\
  Qu\'ebec, QC, Canada \\
  \texttt{abdelali.bouyahia.1@ulaval.ca} \\
  \And 
  Frédéric LeBlanc \\
Institut intelligence et données \\
Université Laval \\
Québec, G1V 0A6, Canada \\
frederic.leblanc@iid.ulaval.ca \\
   \And
 Mario Marchand \\
  Département d'informatique et génie logiciel\\
  Universit\'e Laval\\
  Qu\'ebec, QC, Canada \\
  \texttt{mario.marchand@ift.ulaval.ca} 
  \\
}

\begin{document}
\maketitle

\begin{abstract}
Data augmentation is one of the most widely used techniques to improve generalization in modern machine learning, often justified by its ability to promote invariance to label-irrelevant transformations. However, its theoretical role remains only partially understood. In this work, we propose an information-theoretic framework that systematically accounts for the effect of augmentation on generalization and invariance learning. Our approach builds upon mutual information–based bounds, which relate the generalization gap to the amount of information a learning algorithm retains about its training data. We extend this framework by modeling the augmented distribution as a composition of the original data distribution with a distribution over transformations, which naturally induces an orbit-averaged loss function. Under mild sub-Gaussian assumptions on the loss function and the augmentation process, we derive a new generalization bound that decompose the expected generalization gap into three interpretable terms: (1) a distributional divergence between the original and augmented data, (2) a stability term measuring the algorithm’s dependence on training data, and (3) a sensitivity term capturing the effect of augmentation variability. Tighter bound is then developed via individual sample and augmentation mutual information. To connect our bounds to the geometry of the augmentation group, we introduce the notion of group diameter, defined as the maximal perturbation that augmentations can induce in the input space. The group diameter provides a unified control parameter that bounds all three terms and highlights an intrinsic trade-off: small diameters preserve data fidelity but offer limited regularization, while large diameters enhance stability at the cost of increased bias and sensitivity. We validate our theoretical bounds with numerical experiments, demonstrating that it reliably tracks and predicts the behavior of the true generalization gap.
\end{abstract}

\section{Introduction}
\label{sec:introduction}
Many tasks in machine learning exhibit invariance to certain sets of transformations of the input data. A model is invariant if its output does not change when its input is transformed in ways that are irrelevant to the task. In particular, labels of images remain invariant under translations and rotations, and labels of graphs are identical under nodes permutations. In recent years, various techniques and models have been developed to capture invariance, with the belief that incorporating invariance into models improves sample complexity and generalization capability. This intuition is supported by several empirical results ranging from image classification \citep{fort2021drawing, cohen2018spherical, cohen2019general} and set learning \citep{zaheer2017deep, qi2017pointnet} to graph represenation learning \citep{kipf2016semi}.

Designing invariant neural network architectures is an approach that explicitly embeds invariance in models \citep{bronstein2021geometric}. For instance, Convolutional Neural Networks (CNNs) have been widely used to achieve translation invariance in the field of computer vision \citep{fukushima1982neocognitron, lecun1998gradient}, while Group Equivariant CNNs (G-CNNs) have been introduced to incorporate invariance to more transformations such as rotations \citep{cohen2016group, cohen2018spherical} and scales \citep{worrall2019deep, sosnovik2019scale}. However, embedding invariances in model architectures is task-specific and challenging in general and incurs high computational costs \cite{he2021efficient}, which makes them more difficult to scale up. This is the case of G-CNNs which introduce an extra dimension to encode each additional transformation type. Moreover, more recent works on protein folding \citep{abramson2024accurate} and conformer generation \citep{wang2023generating} found that incoprarating invariance in the architecture by design is not required and opted for non-invariant models and data augmentation as an alternative.


Instead of designing new architectures, data augmentation \citep{shorten2019survey} is a straightforward and universally applicable approach to train invariant models. It consists of adding  label-invariant transformed copies of the available data points into the training set and training a model with the augmented data. This practice has a long history, and as mentioned by \cite{chen2020group} related ideas date back at least to \cite{baird1992structured}. Since then, several works have been carried out to develop efficient strategies of data augmentation, at least in computer vision \citep{shorten2019survey, yang2022image}, speech recognition \citep{ko2015audio}, and natural language processing \citep{feng2021survey}. Specifically, it is used to train state of the art models in a variety of tasks such as image classification \citep{krizhevsky2017imagenet, lafarge2021rotation, gu2022detecting, cubuk2021tradeoffs}. Although data augmentation has shown empirical success, especially in settings with limited labeled data or where models are susceptible to overfitting, theoretical understanding of why and how augmentation enhances generalization remains limited. To address this gap, this paper investigates data augmentation through the lens of information-theoretic generalization guarantees. 

\textbf{Contributions.} We propose an information-theoretic framework to analyze the generalization properties of models trained with data augmentation. Our approach builds upon mutual information-based generalization bounds \citep{xu2017information, bu2020tightening}, which relate the expected generalization gap to the amount of information a learned model retains about its training data. We extend this framework to account for augmentation by modeling the augmented data distribution as a composition of the original data distribution and a distribution over the set of transformations. This leads to an orbit-averaged loss function, in which the loss of each training example is replaced by a distributional average over its augmented variants. Under mild sub-Gaussian assumptions on the loss function and the augmentation process, we derive generalization bounds that decompose the expected generalization gap into interpretable terms: (1) a divergence between the original data distribution and the augmented one, (2) a mutual information term reflecting the algorithm’s dependence on the original data, and (3) a novel term that quantifies sensitivity to the choice of augmentations. These terms allow us to explicitly study how augmentation affects generalization. This corresponds to Theorem~\ref{th:xu_like_bound_mi} and Theorem~\ref{th:bu_like_bound_mi}

To capture the geometry of the augmentation group, we introduce the notion of the group diameter, defined as the maximal displacement that augmentations can induce in the input space (Equation \ref{eq:def_delta}). The group diameter serves as a natural control parameter for bounding the three terms above: it upper bounds the distribution shift between the original and augmented distributions (Proposition \ref{prop:bound_first_term}), constrains the effective mutual information leakage in the stability term (Proposition \ref{prop:diff_mi}), and directly controls the magnitude of the augmentation sensitivity term (Proposition \ref{prop:bound_third_term} and Proposition \ref{prop:bound_third_term_count_setting}). In particular, Proposition~\ref{prop:bound_third_term} formalizes a key intuition long acknowledged in the literature \citep{chen2020group, dao2019kernel, balestriero2022data}, that data augmentation promotes invariance and regularization, by rigorously expressing this effect as a bound on the augmentation-induced mutual information. This unified perspective highlights the geometric role of augmentations and provides explicit, interpretable conditions under which augmentation improves generalization.

Moreover, our analysis reveals that the group diameter introduces a fundamental trade-off in the effect of data augmentation on generalization. A small group diameter ensures that the augmented distribution remains close to the original one and that augmentation sensitivity is minimal, but it provides limited regularization since the learner still retains significant information about individual training examples. Conversely, a large group diameter can substantially reduce the mutual information between the learned model and the training data, thereby improving stability, but at the cost of introducing distributional bias and increasing sensitivity to the choice of augmentation. This trade-off highlights the role of the group diameter as a natural control parameter: moderate values yield the best balance between stability and fidelity to the data distribution, providing a principled explanation for the empirical observation that mild augmentations improve generalization while overly aggressive ones degrade performance.

In summary, our framework establishes a principled connection between group-theoretic properties of augmentations, information-theoretic generalization guarantees, and empirical robustness of learning algorithms. By introducing the group diameter and demonstrating how it bounds all three terms, we provide both theoretical insights and practical tools for understanding the benefits and limitations of data augmentation.
To ground our theoretical findings, we provide concrete examples and experiments on datasets such as MNIST \citep{lecun1998mnist} and FashionMNIST \citep{xiao2017fashionmnist} with affine augmentations, including small rotations and translations. The results illustrate how our bounds behave in practice, showing that augmentations reduce the mutual information between the training data and the learned model parameters.

\section{Related work}
\label{sec:related_work}
The literature on data augmentation for invariance learning generally falls into two main categories: work that develops efficient augmentation strategies that improve empirical performance \citep{lim2019fast, cubuk2020randaugment, cubuk2019autoaugment, cubuk2021tradeoffs, hounie2023automatic}, and work that seeks to provide a theoretical understanding of data augmentation, analyzing how it influences learning dynamics and generalization guarantees. Foundational contributions in this direction study augmentation from perspectives such as stability, algorithmic robustness, and sample complexity. In this section, we primarily focus on the latter body of work, reviewing theoretical frameworks that formalize the effects of data augmentation on learning and shed light on the mechanisms through which it improves generalization.

\textbf{Theoretical understanding of data augmentation.} Theoretical understanding of data augmentation has received significant attention in recent years and is still being developed. One line of work develops an understanding of data augmentation as a form of regularization. \citep{hernandez2018data} show that data augmentation induces implicit regularization and achieves higher performance compared to models trained with weight decay and dropout \citep{srivastava2014dropout}.  \citep{balestriero2022data} derive the first explicit regularizer that corresponds to data augmentation, allowing them to quantify its benefits and limitations. \citep{botev2022regularising} propose an explicit regularizer that encourages invariance and show empirically that it leads to improved generalization. Later, \citep{yang2023sample} confirm this improvement theoretically for linear regression and two-layer neural networks. \citep{geiping2022much} show that data augmentation implicitly encourages flatter minima, serving as regularizers whose effect diminishes as the number of training samples increases. \citep{lin2024good} demonstrate that data augmentation induces implicit spectral regularization, which improves generalization. \citep{lejeune2019implicit} popose a model in which data augmentation increases the smoothness of neural networks.  Another line of work adopts a group-theoretic formulation and views data augmentation as invariance learning by averaging over group actions. \citep{chen2020group} propose a general group-theoretic framework for data augmentation and demonstrate that averaging reduces Rademacher complexity \citep{bartlett2002rademacher} and induces better generalization through variance reduction. \citep{lyle2020benefits} investigate PAC-Bayes theory to show the reduction of generalization upper bound. However, their framework is restricted to invariant distributions and considers only augmentations that preserve the original data distribution. \citep{shao2022theory} show that data augmentation outperforms vanilla empirical risk minimization and introduce a new complexity measure to characterize PAC learnability with data augmentation. Additionally, other works study the benefits of some specific types data augmentation. \citep{dao2019kernel} formulate data augmentation as a Markov process, revealing an asymptotic connection between the Bayes-optimal classifier and a kernel classifier that depends on the augmentation scheme. However, their work does not offer a quantitative study of the generalization performance of such classifier. \citep{shen2203data} show  that training a two-layer convolutional neural network with data augmentation can be viewed as a feature manipulation mechanism.


\textbf{Data augmentation is not without risk.} While data augmentation is widely used to enhance generalization and encourage invariance learning, a growing body of work has highlighted its potential downsides. \citep{xu2020wemix} demonstrate that data augmentation may amplify existing data biases, resulting in degraded model performance on the original distribution. \citep{lin2024good} demonstrate that the increase in bias introduced by data augmentation can offset the reduction in variance, potentially limiting its overall benefit. \citep{balestriero2204effects} reveal that augmentations can degrade performance on certain classes while improving it on others, highlighting a trade-off in class-level accuracy. Building on this, \citep{kirichenko2401understanding} propose a framework to analyze how data augmentation interacts with class-level learning dynamics.  \citep{kapoor2022uncertainty} argues that data augmentation can lead models to misrepresent uncertainty. Additionally, \citep{shao2022theory} demonstrate that distinguishing between the original data and the transformed data is necessary to achieve optimal accuracy, implying that algorithms treating them equivalently (including standard data augmentation) may be suboptimal or even harm model performance.


\textbf{Information-theoretic learning bound} 
The information-theoretic approach is a recently developed framework that offers guarantees on the generalization performance of machine learning algorithms. One major advantage of this framework is its ability to consider all aspects of the learning problem, including the data distribution, the hypothesis space, the learning algorithm, and the loss function; elements that cannot be fully captured by conventional complexity-based bounds such as  Vapnik–Chervonenkis dimension \citep{vapnik1999nature, shalev2014understanding}  and Rademacher complexity \citep{bartlett2002rademacher, mohri2018foundations} or other notions of algorithmic stability such as uniform stability \citep{bousquet2002stability, bousquet2020sharper}, and  differential privacy \citep{dwork2015preserving}. Building on the work in \cite{russo2016controlling}, \cite{xu2017information} derived a generalization bound based on the mutual information between the learning algorithm's output and its training dataset. Several enhancements to this bound have been explored to achieve tighter bounds. \citep{bu2020tightening} provided a tightened bound based on the mutual information between the output of the algorithm and each data point in the training dataset, while the bound  in \citep{asadi2018chaining} leveraged chaining mutual information techniques to tighten the bound presented in \citep{xu2017information}. The bound introduced in \citep{negrea2019information} incorporate the mutual information between the algorithm's output and a random set of the dataset. Most recently, \citep{steinke2020reasoning} proposed to use the Conditional Mutual Information (CMI) to tighter the bound. Some works \citep{pensia2018generalization, negrea2019information} applied the mutual information framework to derive generalization bounds for noisy and iterative algorithms such as Stochastic Gradient Langevin Dynamics (SGLD).

\section{Preliminaries}
\label{sec:preliminaries}
\subsection{Notation}
We adopt the convention of using uppercase letters, e.g., $X, Y$, to denote random variables, lowercase letters $x, y$ to denote their realizations, and calligraphic letters $\mathcal{X}, \mathcal{Y}$ to denote the spaces on which they are defined. We let $D_{KL}(P \Vert Q)$ denote the Kullback–Leibler (KL) divergence between two probability measures $P$ and $Q$. This quantity is only defined if $P$ is absolutely continuous with respect to $Q$, i.e., $Q(A)=0$ implies $P(A)=0$ for every $Q$-measurable subset $A$. For two random variables $X$ and $Y$ with joint distributions denoted $P_{X,Y}$ and respective marginals denoted $P_X$ and $P_Y$, we let $I(X; Y ) \eqdef D_{KL}(P_{X,Y} \Vert P_X \otimes P_Y )$ denote their mutual information. If $Z$ is a third random variable, then the disintegrated mutual information between $X$ and $Y$ given $Z$ is defined as $I^Z(X; Y) \eqdef D_{KL}(P_{X,Y|Z} \Vert P_{X|Z} \otimes P_{Y|Z} )$ \citep{negrea2019information}, and the corresponding conditional mutual information is defined as $I(X; Y | Z ) \eqdef  \underset{Z}{\mathbb{E}} \left[I^Z(X; Y) \right]$. We use $\log$ to denote the natural logarithm (i.e., base $e \approx 2.718$). Thus, the units for all information-theoretic quantities are nats rather than bits.

\subsection{Information-theoretic genaralization bounds}
Let $\mathcal{X}$ be an input space, $\mathcal{Y}$  the corresponding output space and $\mathcal{D}$ an unknown distribution on the product space $\mathcal{Z} \eqdef \mathcal{X} \times \mathcal{Y}$. Let $\mathcal{H} \eqdef \{h \mid h:\mathcal{X} \mapsto \mathcal{Y}\}$ be a hypothesis class parameterized by elements from some set $\mathcal{W} \subseteq \mathbb{R}^p$. Given a non negative loss function $\ell : \mathcal{W} \times \mathcal{Z} \to \mathbb{R}^+ $ and a sample $S = \left \{Z_i \right \}_{i=1}^m \overset{i.i.d.}{\sim} \mathcal{D}^m$, the true risk and the empirical risk of $w \in \mathcal{W}$ are respectively defined as $L_{\mathcal{D}}(w) \eqdef \underset{Z \sim \mathcal{D} }{\mathbb{E}}  \, [\ell (w, Z)]$ and $L_{S}(w) \eqdef \frac{1}{m} \sum_{i=1}^{m} \ell (w, Z_i)$.

In the context of information-theoretic analysis, a learning algorithm $\mathcal{A}$ is viewed as a randomized mapping that takes a dataset $S$ as input and outputs a hypothesis $W$ sampled according to a conditional distribution $P_{W|S}$, i.e. $W = \mathcal{A}(S) \sim P_{W|S} $.  A key strength of this framework lies in its capacity to account for all components of the learning problem, including the data distribution, hypothesis space, learning algorithm, and loss function. Building on the work of \cite{russo2016controlling}, \cite{xu2017information} derive a generalization bound based on the mutual information $I(S;W)$ between the input dataset $S$ and the output $W$ of the learning algorithm, under the assumption that the loss $\ell(w,Z)$ is a sub-Gaussian random variable for any $w\in\mathcal{W}$. In contrast with \citep{xu2017information}, we use the notation $I_\ell(S;W)$ to emphasize that the mutual information is evaluated with respect to the loss function $\ell$. Although the loss function does not appear explicitly in the definition of the mutual information, it shapes the learning algorithm, which determines the distribution of the output $W$ given the training data $S$, and thus implicitly affects the value of $I(S;W)$.

\begin{definition} 
\label{def:cgf}
The cumulant generating function (CGF) of a random variable $X$ is defined as
$$\psi_X (\lambda) \eqdef \log   \mathbb{E}[e^{\lambda(X - \mathbb{E}X)}].$$
\end{definition}
If it is well defined, the CGF $\psi_X (\lambda)$ is convex and it satisfies $\psi_X (0) = 0$ and $\psi'_X (0) = 0$.
\begin{definition} 
\label{def:sub_gaussian}
A random variable $X$ is said to be $R$-sub-Gaussian if its CGF is bounded as 
$$\psi_X (\lambda) \leq \frac{\lambda^2 R^2}{2}, \hspace{1cm} \forall \lambda \in \mathbb{R}.$$
\end{definition}
As a special case, if $X$ is bounded such that the inequality $-\infty < a \leq X \leq b < \infty$ holds for
some constants $a$ and $b$, then $X$ is $(b - a)/2$-sub-Gaussian. Also if $X \sim \mathcal{N}(0, \sigma^2)$, then $X$ is $\sigma$-sub-Gaussian

In this context, the expected generalization gap of the learning algorithm $\mathcal{A}$ is defined as the expected difference between the true risk $L_{\mathcal{D}}(W)$ and the empirical risk $ L_{S}(W)$, where the expectation is taken with respect to the joint distribution $P_{S,W} = \mathcal{D}^m \otimes P_{W|S}$
\begin{equation}\label{eq:gen}
    \mathrm{gen}(\mathcal{D}, P_{W|S}) = \underset{S \sim \mathcal{D}^m}{\mathbb{E}} \underset{W\sim P_{W|S}}{\mathbb{E}} [L_{\mathcal{D}}(W) - L_{S}(W)]. 
\end{equation}
This expected generalization gap can be bounded as follows:
\begin{theorem} [\cite{xu2017information}]
\label{th:raginsky_bound}
Suppose $\ell(w,Z)$ is $R$-sub-Gaussian under $Z \sim \mathcal{D}$  for every $w \in \mathcal{W}$, then for any learning algorithm characterized by $P_{W|S}$ such that $S \sim \mathcal{D}^m $, we have
\begin{equation}
\label{bound:raginsky_bound}
  |\mathrm{gen}(\mathcal{D}, P_{W|S})| \leq \sqrt{\frac{2R^2}{m}I_{\ell}(S;W)}.  
\end{equation}
\end{theorem}
The bound (\ref{bound:raginsky_bound}) suggests that reducing the mutual information between the training data $S$ and the learned model $W = \mathcal{A}(S)$ can help mitigate overfitting. Indeed, if $I_{\ell}(S;W) \leq \epsilon$, the learning algorithm $\mathcal{A}$ is said to be $(\epsilon, \mathcal{D})$-stable in input-output mutual information \citep{xu2017information}. In this case, the bound in \eqref{bound:raginsky_bound} evaluates to $\sqrt{2R^2\epsilon / m}$.

However, for many deterministic algorithms where $W$ is a deterministic function of $S$, the bound (\ref{bound:raginsky_bound}) becomes vacous as we can have $I_{\ell}(S, W) = \infty $ when $W$ is a continuous random variable.  To address this, \citep{bu2020tightening} proposed a refinement based on individual sample mutual information  $I_{\ell}(Z_i, W)$ between the output parameter $W$ and each sample $Z_i$ of the training set $S$, leading to tighter and more informative generalization bounds. By decomposing the mutual information at the sample level, this approach offers a finer-grained analysis of how specific training examples influence the learned model, thereby enabling more precise characterizations of generalization behavior. Indeed, they observed that the expected generalization gap in Equation~\eqref{eq:gen} can be  written as 
\begin{align*}
        \mathrm{gen}(\mathcal{D}, P_{W|S}) &= \underset{S \sim \mathcal{D}^m}{\mathbb{E}} \underset{W\sim P_{W|S}}{\mathbb{E}} [L_{\mathcal{D}}(W) - L_{S}(W)]\\
        &= \frac{1}{m} \sum_{i=1}^m \left( \underset{\tilde{Z} \sim \mathcal{D}}{\mathbb{E}}\, \underset{W\sim P_{W}}{\mathbb{E}} [\ell(W,\tilde{Z})] - \underset{Z_i \sim \mathcal{D}}{\mathbb{E}}\, \underset{W\sim P_{W|Z_i}}{\mathbb{E}} [\ell(W,Z_i)] \right);
\end{align*}
where $W$ and $Z_i$ are dependent, and $W$ and $\tilde{Z}$ are independent. By bounding each term, \citep{bu2020tightening} derived the following bound
\begin{theorem} [\cite{bu2020tightening}]
\label{th:bu_bound}
Suppose $\ell(w,Z)$ is $R$-sub-Gaussian under $Z \sim \mathcal{D}$  for all $w \in \mathcal{W}$, then for any learning algorithm characterized by $P_{W|S}$, for $S \sim \mathcal{D}^m $, we have 
\begin{equation}
\label{bound:bu_bound}
  |\mathrm{gen}(\mathcal{D}, P_{W|S})| \leq \frac{1}{m} \sum_{i=1}^m\sqrt{2R^2I_{\ell}(Z_i;W)}.  
\end{equation}
\end{theorem}
Using the chain rule of mutual information and Jensen’s inequality, we can show that the bound (\ref{bound:bu_bound}) is tighter than (\ref{bound:raginsky_bound}), that is,
$$ \frac{1}{m} \sum_{i=1}^m\sqrt{2R^2I_{\ell}(Z_i;W)} \leq \sqrt{\frac{2R^2}{m}I_{\ell}(S;W)}.$$
This inequality highlights the fact that controlling generalization through per-sample mutual information terms yields a tighter characterization than doing so through a dataset level mutual information. From a learning-theoretic perspective, this indicates that bounding generalization at the level of individual samples reflects the fine-grained variability in how each example shapes the learned model, as opposed to treating the dataset as a single undifferentiated unit. This finer-grained view parallels the intuition from curriculum learning \citep{soviany2022curriculum, zhang2025fastdinov2}: just as curriculum learning emphasizes that different examples affect the training dynamics with unequal strength, the per-sample mutual information bound provides a rigorous formalization of this principle by quantifying the uneven impact of data points on generalization.

\section{Problem setup}

In this section, we begin by formalizing the concept of invariance in a mathematical framework. We then present data augmentation as a widely used approach to promote invariance in learning algorithms. Finally, we introduce a definition of the generalization gap that is specific to learning with augmented data, which serves as the basis for our theoretical analysis.

\subsection{Invariance}
Invariance is a fundamental concept in machine learning that aims to ensure that the predictions of a model remain unchanged under a specific set of transformations of the input data. For instance, in image classification, rotating or slightly shifting an image of a cat should not cause a model to misclassify it as a different object. For mathematical convenience, we consider the case where the set of transformations to which we would like to be invariant for a given task forms a group, which is a classic setting studied in the literature \citep{lyle2020benefits, chen2020group}. This group-theoretic perspective provides a structured way to reason about invariance and has informed both theoretical \cite{chen2020group} and architectural developments in the field \citep{cohen2016group, bronstein2021geometric}. In fact, the emerging field of geometric deep learning is entirely based on group theory to derive different inductive biases and network architectures that implement them from the first principles of invariance \citep{bronstein2021geometric}. Our results can be extended to non-group transformations. 

Let $\mathcal{G}$ be a compact group of transformations, such as the translation or the permutation group, that act on the input space $\mathcal{X}$. A group action of $\mathcal{G}$ on the space $\mathcal{X}$ is a mapping $\phi : \mathcal{G} \times \mathcal{X} \to \mathcal{X} $ satisfying two properties for all $g_1, g_2 \in \mathcal{G}$ and $x \in \mathcal{X}$: $\phi(e,x) = x$, where $e$ is the identity element in $\mathcal{G}$, and $\phi(g_1g_2,x) = \phi(g_1,\phi(g_2,x))$. For notational simplicity, we let $\phi(g,x) \eqdef gx$, and we say that $gx$ is the transformation of $x$ induced by $g$. The orbit of a point $x \in \mathcal{X}$ under the action of $\mathcal{G}$ is the set $O_x \eqdef \{ gx \mid g \in \mathcal{G}\}$. This set contains all transformed versions of $x$ under $\mathcal{G}. $\cite{milne2013group} provides more background on group theory.  

We now formalize the notion of invariance of a function with respect to a group action.


\begin{definition}
 A function $h:\mathcal{X} \mapsto \mathcal{Y}$ is said to be $ \mathcal{G}-$\textit{invariant} if $h(gx) = h(x)$ for all $g\in \mathcal{G}$ and $x\in \mathcal{X}$.
\end{definition}

In other words, this property implies that $h$ is constant on each orbit, meaning its output remains unchanged under any transformation induced by the group $\mathcal{G}$. A notable example of invariance is permutation invariance, which naturally arises in set-strutured data such as object bounding boxes \citep{carion2020end}, point clouds \citep{qi2017pointnet}, support sets in meta-learning \citep{finn2017model} and bags in multi-instance learning \citep{ilse2018attention}. The defining structural property of sets is the absence of any assumed ordering among their elements. Thus, functions operating on sets must be invariant to permutations of their inputs, that is, they should produce the same output regardless of the ordering of set elements. This principle has motivated the development of architectures such as Deep Sets \citep{zaheer2017deep}, which explicitly enforce permutation invariance, as well as attention-based models that aggregate information in a symmetric manner to respect set structure \citep{lee2019set}.

We also have the following definition of invariant distributions
\begin{definition}
A distribution $\mathcal{D}$ over $\mathcal{X} \times \mathcal{Y}$ is said to be $\mathcal{G}$-invariant if for all measurable sets $A \subseteq \mathcal{X} \times \mathcal{Y}$ and all $g \in \mathcal{G}$, we have $\mathcal{D}(gA) = \mathcal{D}(A)$; where $gA \eqdef \{ (gx, y) \mid (x, y) \in A\}$.  
\end{definition} 


Intuitively, a $\mathcal{G}$-invariant distribution treats every point in an orbit equally, that is, it doesn't change if a point is moved around using the group action. Moreover, its marginal distribution $\mathcal{D}_\mathcal{X}$ over the input space $\mathcal{X}$ admits a disintegration into two components: a distribution $\mu$ over the space of orbits, and a conditional distribution $\lambda_{O}$, typically uniform (e.g., induced by the normalized Haar measure), on each orbit $O$ \citep{bloem2020probabilistic}. In other words, sampling from a $\mathcal{G}$-invariant distribution can be seen as a two-stage process: first, draw an orbit $O$ from the distribution $\mu$, then sample a point from that orbit according to $\lambda_{O}$. Let $\mathcal{X}/\mathcal{G}$ denote the orbit space, and let $\pi : \mathcal{X} \to \mathcal{X}/\mathcal{G}$ be the canonical projection that maps each $x \in \mathcal{X}$ to its corresponding orbit $O_x$. The distribution $\mu$ can then be defined as the pushforward of $\mathcal{D}_\mathcal{X}$ on the orbit space $\mathcal{X}/\mathcal{G}$ via $\pi$, i.e., $\mu(B) = \mathcal{D_\mathcal{X}}(\pi^{-1}(B))$ for all measurable sets $B \subseteq \mathcal{X}/\mathcal{G}$. This leads to the following integral representation of the marginal distribution
$\mathcal{D}_\mathcal{X}$ 
$$\mathcal{D}_\mathcal{X}(A) = \int_{\mathcal{X}/\mathcal{G}} \lambda_O(A \cap O) \,d\mu(O)$$ 
for all measurables sets $A \subseteq \mathcal{X} $.

\subsection{Generalization gap for data augmentation}
Let $\mathcal{G}$ be the group of transformations with respect to which we aim to achieve invariance, and let $\mathcal{M}_\mathcal{G}$ denote the set of all probability measures over $\mathcal{G}$. To learn a $\mathcal{G}$-invariant predictor, data augmentation uses transformations $G$ sampled from a distribution $\mathcal{D}_\mathcal{G} \in \mathcal{M}_\mathcal{G}$ to generate new samples $GZ_i \eqdef (GX_i, Y_i)$. This procedure leads to the minimization of the following objective

\begin{equation}\label{eq:aug}
    L_{\mathcal{D}_\mathcal{G} \circ S}(w) \eqdef \frac{1}{m} \sum_{i=1}^{m}  \underset{G \sim \mathcal{D}_\mathcal{G} }{\mathbb{E}}\ell (w, GZ_i)\ 
\end{equation}

If the augmentation distribution $\mathcal{D}_\mathcal{G}$ is not chosen adequately, $L_{\mathcal{D}_\mathcal{G} \circ S}$ can be a poor estimate of the expected risk $L_{\mathcal{D}}$, introducing side effects that exceed the benefit of data augmentation \citep{hounie2023automatic}. We highlight this phenomenon more rigorously in Section \ref{sec:bound_augment_empir}.  One widely used distribution is the normalized Haar measure $\lambda_\mathcal{G}$, often regarded as the uniform distribution over $\mathcal{G}$. Although foundational in many theoretical works \citep{chen2020group, elesedy2021provably}, its practical use is not without risk. Indeed, in image classification, aggressive translations and certain rotations of images, can produce examples that can harm the learning task. That is, after transformation, the object of interest may be moved out of the boundary and lost or change the label as some examples of different labels can be approximately transformed into each other  (e.g., digits "6" and "9" in the MNIST dataset). Hence, in practice we must choose a distribution $\mathcal{D}_\mathcal{G}$ that encode our prior knowledge about the nature of the data and the problem at hand. Recently, \cite{hounie2023automatic} proposed to recover automatically an augmentation distribution which adapts to the learning task by formulating data augmentation as an invariance constrained learning problem. In addition, \cite{miao2022learning} propose a method for automatically learning input-specific augmentations from data.

Unfortunately, many important problems in machine learning exhibit invariance to transformation groups that are either combinatorially large (e.g., the permutation group $\mathcal{G} = \Sigma_n$ in unordered set learning \citep{zaheer2017deep}) or continuous and infinite (e.g., the special orthogonal group $\mathcal{G} = SO(2)$ in histopathology image analysis \citep{lafarge2021roto}). In such cases, computing the expectation $\mathbb{E}_{G \sim \mathcal{D}_\mathcal{G} }\ell (w, GZ_i)$ becomes computationally intractable or even infeasible. A common practical alternative to the objective in Equation~\ref{eq:aug} is to augment the training set $S$ by applying a finite, randomly sampled set of transformations $E \eqdef \{G_{j}\}_{j=1}^{n} \sim \mathcal{D}_\mathcal{G}^n$ to each training example $Z_i$. This results in an augmented dataset $S_E \eqdef \{ \{ G_{j}Z_i \}_{j=1}^{n} \}_{i=1}^{m}$, upon which a standard empirical risk minimization algorithm, or variants, is performed by minimizing the following objective

\begin{equation}\label{eq:aug_train}
    L_{E \circ S }(w) \eqdef \frac{1}{mn} \sum_{i=1}^{m} \sum_{j=1}^{n} \ell (w, G_{j}Z_{i});
\end{equation}

This formulation approximates the ideal $\mathcal{G}$-invariant risk in Equation~\ref{eq:aug} by averaging the loss over a finite set of transformations, thus reducing computational cost while still encouraging the model to learn invariances. While this approach is widely used in practice due to its scalability, it introduces approximation error, as the distribution $\mathcal{D}_\mathcal{G}$ over the group is only partially explored. The choice of $n$ (number of augmentations) and the sampling strategy $\mathcal{D}_\mathcal{G}$ therefore play a crucial role in determining the effectiveness of the resulting invariance. 


\textbf{Expected generalization gap for data augmentation.} Unlike the standard learning setting, data augmentation introduces an additional source of randomness through the set of transformations $E$, sampled from the distribution $\mathcal{D}_\mathcal{G}$. These transformations are used to generate $n$ augmented versions of each training example $Z_i$, resulting in an expanded datanset $S_E$. Consequently, the learning algorithm receives $S$ and the $E$ to produce a predictor $W = \mathcal{A}(S,E) \sim P_{W|S,E}$. In this setting, the expected generalization gap is defined as follows

\begin{equation}\label{eq:exp_gen_gap_aug}
    \mathrm{gen}(\mathcal{D}, P_{W|S,E}) \eqdef \underset{S \sim \mathcal{D}^m}{\mathbb{E}} \underset{E \sim \mathcal{D}_\mathcal{G}^{n}}{\mathbb{E}} \underset{W\sim P_{W|S,E}}{\mathbb{E}} [L_{\mathcal{D}}(W) - L_{E \circ S }(W)];
\end{equation}
where the expectation is with respect to the joint distribution $P_{S,E,W} = \mathcal{D}^m \otimes \mathcal{D}_\mathcal{G}^{n} \otimes P_{W|S,E}$. 

The information-theoretic generalization bounds presented in Section~\ref{sec:preliminaries} rely on two key assumptions: the sub-Gaussianity of the loss function and the i.i.d.\ nature of data drawn from the distribution $\mathcal{D}$. The i.i.d. assumption plays a pivotal role in the derivation of resulting bounds, as it allows for establishing the sub-Gaussian concentration of the empirical risk. This, combined with the decoupling estimate Lemma~\ref{lem:decoupling-estimate-lemma}, leads to the derivation of the stated bounds. However, data augmentation violates this assumption by introducing dependencies among augmented samples. Moreover, it alters the original data distribution $\mathcal{D}$, further complicating the analysis. As a result, deriving generalization bounds for invariance learning using data augmentation becomes more challenging and requires more assumptions. In the next session, we discuss how the input-output mutual information of the learning algorithm can be used to derive a bound on the expected generalization gap defined in Equation \ref{eq:exp_gen_gap_aug}.

\section{Information-theoretic generalization bounds for data augmentation}
\label{sec:results}
In this section, we derive two novel mutual information based bounds on the expected generalization gap defined in Equation (\ref{eq:exp_gen_gap_aug}). To set the stage, we first explore how data augmentation affects the data distribution by expanding its support via random transformations.  

\subsection{Data augmentation induce new distribution}
When data augmentation is performed using transformations sampled from a distribution $\mathcal{D}_\mathcal{G}$, the original data distribution $\mathcal{D}$ is modified through support expansion and probability mass redistribution, yielding a new distribution denoted by $\mathcal{D}_\mathcal{G} \circ \mathcal{D}$. Formally, the augmented data distribution is defined as
$$ \mathcal{D}_\mathcal{G} \circ \mathcal{D} \eqdef \underset{G \sim \mathcal{D}_\mathcal{G} }{\mathbb{E}} [G_* (\mathcal{D})] $$
where $G_* (\mathcal{D})$  denotes the pushforward of the original distribution $\mathcal{D}$ by transformation $G$, i.e., the distribution of  the transformed random variable  $GZ$ when $Z\sim \mathcal{D}$. 

More explicitly, for all measurable sets $A \subseteq \mathcal{X} \times \mathcal{Y}$, the augmented distribution satisfies
$$ \mathcal{D}_\mathcal{G} \circ \mathcal{D} (A) =  \underset{G \sim \mathcal{D}_\mathcal{G} }{\mathbb{E}} [\mathcal{D}(\{Z \in \mathcal{X} \times \mathcal{Y} \mid GZ \in A \})].$$
If the group $\mathcal{G}$ contains transformations that move data points outside the original support of $\mathcal{D}$, the support of $\mathcal{D}_\mathcal{G} \circ \mathcal{D}$ is strictly larger. The invariance of $\mathcal{D}$ to $\mathcal{G}$ (that is, $G_* (\mathcal{D}) = \mathcal{D}$ for all $G \sim \mathcal{D}_\mathcal{G}$) is a special case in which the support remains unchanged and $\mathcal{D}_\mathcal{G} \circ \mathcal{D} = \mathcal{D}$.

Given a predictor $w \in \mathcal{W}$ and a loss function $\ell$, the augmented expected risk $L_{ \mathcal{D}_\mathcal{G} \circ \mathcal{D}}(w)$ under the new distribution $\mathcal{D}_\mathcal{G} \circ \mathcal{D}$ is defined as
$$L_{ \mathcal{D}_\mathcal{G} \circ \mathcal{D}}(w) \eqdef \underset{Z \sim \mathcal{D}_\mathcal{G} \circ \mathcal{D}}{\mathbb{E}} \, [\ell (w, Z)] = \underset{Z' \sim \mathcal{D} }{\mathbb{E}} \, \underset{G \sim \mathcal{D}_\mathcal{G} }{\mathbb{E}} \, [\ell (w, GZ')] \eqdef \underset{Z' \sim \mathcal{D} }{\mathbb{E}} [\ell_\mathcal{G}(w, Z')];$$
where $\ell_\mathcal{G}(w, Z') \eqdef \underset{G \sim \mathcal{D}_\mathcal{G} }{\mathbb{E}} \, \ell (w, GZ')$ for $Z'\sim \mathcal{D}$ is the orbit-averaged loss function. 

The orbit-averaged loss function evaluates how well the model performs on average across the entire orbit of a single data point $Z'\sim \mathcal{D}$, with transformations drawn from the distribution $\mathcal{D}_\mathcal{G}$. This loss encourages the model to be consistent no matter how the input is transformed, pushing it toward learning representations that are $\mathcal{G}$-invariant. Accordingly, the augmented expected risk $L_{ \mathcal{D}_\mathcal{G} \circ \mathcal{D}}(w)$  quantifies the expected performance of a predictor $w$ when its inputs are randomly transformed, reflecting the model’s ability to generalize under data transformations. This is particularly important in realistic settings, where inputs may vary due to factors such as lighting conditions, rotations, or translations. Therefore, minimizing the augmented expected risk helps ensure that the model generalizes well across such transformations, ultimately leading to more robust predictions.

\subsection{Information-theoretic generalization bound }
\label{sec:bound_augment_empir}
To better understand the effect of data augmentation, we are interested in studying the expected generalization gap of a learning algorithm $\mathcal{A}$ applied on the augmented dataset $S_E$. Formally, we need to bound $|\mathrm{gen}(\mathcal{D}, P_{W|S,E})|$,
which has been defined in Equation \ref{eq:exp_gen_gap_aug}. To derive such a bound, we first decompose $\mathrm{gen}(\mathcal{D}, P_{W|S,E})$ as follows
\begin{align}
\ \nonumber \mathrm{gen}(\mathcal{D}, P_{W|S,E}) = {}& \underset{S \sim \mathcal{D}^m}{\mathbb{E}} \underset{E \sim \mathcal{D}_\mathcal{G}^{n}}{\mathbb{E}} \underset{W\sim P_{W|S,E}}{\mathbb{E}} [L_{\mathcal{D}}(W) - L_{E \circ S }(W)] \\
     = {}& \underset{S \sim \mathcal{D}^m}{\mathbb{E}} \underset{E \sim \mathcal{D}_\mathcal{G}^{n}}{\mathbb{E}} \underset{W\sim P_{W|S,E}}{\mathbb{E}} [L_{\mathcal{D}}(W) - L_{\mathcal{D}_\mathcal{G} \circ \mathcal{D}}(W)] \label{eq:true_risks_gap} &&\text{Expected distribution shift} \\
      {}& +   \underset{S \sim \mathcal{D}^m}{\mathbb{E}} \underset{E \sim \mathcal{D}_\mathcal{G}^{n}}{\mathbb{E}} \underset{W\sim P_{W|S,E}}{\mathbb{E}} [L_{\mathcal{D}_\mathcal{G} \circ \mathcal{D}}(W) - L_{\mathcal{D}_\mathcal{G} \circ S}(W)] \label{eq:new_loss_gap} &&\text{Expected gen gap w.r.t to $\ell_\mathcal{G}$}\\
      {}& +  \underset{S \sim \mathcal{D}^m}{\mathbb{E}} \underset{E \sim \mathcal{D}_\mathcal{G}^{n}}{\mathbb{E}} \underset{W\sim P_{W|S,E}}{\mathbb{E}} [L_{\mathcal{D}_\mathcal{G} \circ S }(W) - L_{E \circ S }(W)] \label{eq:new_appox_cost} &&\text{Expected $\ell_\mathcal{G}$ approximation cost}
\end{align}  

where all expectations are taken over the joint sampling of training data $S \sim \mathcal{D}^m$, transformation samples $E \sim \mathcal{D}_\mathcal{G}^n$, and model output $W \sim P_{W|S,E}$.

Thus, the generalization gap we aim to bound arises from three distinct sources. The first is a data-independent term that quantifies the discrepancy introduced by altering the original data distribution through augmentation. The second term captures the generalization gap at the orbit level, where learning is performed using the orbit-averaged loss function $\ell_\mathcal{G}$. This term is mainly influenced by the size $m$ of the original training samples $m$. The third term reflects the approximation error incurred when estimating the orbit-averaged loss $\ell_\mathcal{G}$ using a finite size $n$ of randomly sampled transformations. To derive an upper bound on $|\mathrm{gen}(\mathcal{D}, P_{W|S,E})|$, we analyze and bound each of these three components separately. 

We make the following assumptions
\begin{assumption}
\label{assumption_1}
The random variable $\ell(w,Z)$ is $R$-sub-Gaussian under $ \mathcal{D}_\mathcal{G} \circ \mathcal{D}$  for all $w \in \mathcal{W}$.
\end{assumption}
\begin{assumption}
\label{assumption_2}
The random variable $\ell(w,Gz)$ is $R$-sub-Gaussian under $ \mathcal{D}_\mathcal{G}$  for all $w \in \mathcal{W}$ and $z \in \mathcal{Z}$.
\end{assumption}

Assumption \ref{assumption_1} provides control over the tail behavior of the loss function under the augmented distribution, which in turn enables the application of information-theoretic generalization bounds such as those in Theorem~\ref{th:raginsky_bound} and Theorem~\ref{th:bu_bound}. In parallel, Assumption~ \ref{assumption_2} ensure control over the variability introduced by the data augmentation process when applied to individual examples. We note that, in general, Assumption \ref{assumption_1} does not necessarily imply Assumption \ref{assumption_2}, nor vice versa. Nevertheless, there are loss functions for which both assumptions hold naturally, including the case of bounded loss functions. 

Under Assumptions \ref{assumption_1} and \ref{assumption_2}, the following theorem presents an upper bound on the expected generalization gap defined in Equation (\ref{eq:exp_gen_gap_aug})

\begin{theorem}
\label{th:xu_like_bound_mi}
Suppose assumptions \ref{assumption_1} and \ref{assumption_2} are satisfied, then 
$$ 
\lvert \mathrm{gen}(\mathcal{D}, P_{W|S,E}) \rvert \leq R \left( \underbrace{ \sqrt{2D_{\mathrm{KL}} (\mathcal{D}\Vert \mathcal{D}_\mathcal{G} \circ \mathcal{D})}}_\text{Distribution shift} + \underbrace{\sqrt{\frac{2}{m}I_{\ell_{\mathcal{G}}}(S;W)}}_\text{ MI under oribt loss} + \underbrace{\frac{1}{m} \underset{S \sim \mathcal{D}^m}{\mathbb{E}} \left[ \sum_{i=1}^{m}  \sqrt{\frac{2}{n}I_{\ell}^{Z_i}(E;W)}\right] }_\text{Per-example augmentation MI}\right).
$$
\end{theorem}
The proof of Theorem \ref{th:xu_like_bound_mi} is provided in Appendix \ref{proof:xu_like_bound_mi}. 

\textbf{Remark 1.} To achieve a low expected generalization gap on the original data distribution $\mathcal{D}$, Theorem \ref{th:xu_like_bound_mi} highlights three key factors that should be jointly controlled. 1) The first term in the bound, involving the KL-divergence $D_{\mathrm{KL}} (\mathcal{D}\Vert \mathcal{D}_\mathcal{G} \circ \mathcal{D})$, captures the degree of mismatch between the original data distribution and its augmented counterpart. A low KL-divergence implies that the data augmentation scheme does not introduce a large distributional shift and that augmented samples remain representative of the original task. Therefore, in order for data augmentation to benefit generalization, the augmentation strategy encoded by $\mathcal{D}_\mathcal{G}$ should preserve the essential structure of the original distribution. Alternatively, this also means that the learning algorithm should aim to minimize the risk under the augmented distribution, $L_{{\mathcal{D}_\mathcal{G}}\circ\mathcal{D}}(W)$, instead of the risk under the original distribution, $L_\mathcal{D}(W)$. Indeed, when evaluating generalization with respect to $\mathcal{D}_\mathcal{G} \circ \mathcal{D}$, the KL-divergence term vanishes, reflecting the fact that training and testing distributions are aligned under augmentation. 2) The second term involves the mutual information $I_{\ell_{\mathcal{G}}}(S;W)$, which measures the dependency between the training dataset $S$ and the output $W$ of the learning algorithm under the orbit-averged loss function $\ell_\mathcal{G}$. A low value for this mutual information indicates that the algorithm is stable in input-output mutual information: small perturbations in the training data do not drastically affect the learned model. From an information-theoretic perspective, reducing this dependency is crucial for achieving low generalization error, as it prevents overfitting to specific training examples. 3) The final term in the bound involves the average (over training samples) of mutual information $I_{\ell}^{Z_i}(E;W)$ between the augmentation transformations $E$ and the output $W$, but on a per-example basis. This term quantifies how much the model relies on specific augmented views of the data. A small value suggests that the learning algorithm is invariant or at least robust to the particular transformations applied. Hence, to promote generalization, it is desirable for the learning algorithm to not overly rely on the precise form of the augmentations, but instead to extract signal that is invariant across such transformations.  

Together, these insights suggest that effective generalization with data augmentation requires a careful balance: augmentations should be realistic (low KL-divergence), the model should be stable (low dependence on training data), and robust to the randomness in the augmentation process (low dependence on transformation noise).

\textbf{Remark 2 (Connection to classical generalization bound).} Theorem~\ref{th:xu_like_bound_mi} generalizes classical mutual information-based generalization bounds by explicitly accounting for the effect of data augmentation. To see this, consider the special case where the augmentation set $E = \{e\}$ consists solely of the identity element of the transformation group $\mathcal{G}$. In this setting, the augmentation distribution $\mathcal{D}_\mathcal{G}$ becomes a Dirac measure at the identity, and thus the composed distribution $\mathcal{D}_\mathcal{G} \circ \mathcal{D}$ reduces to the original data distribution $\mathcal{D}$. Additionally, the augmentation-dependent mutual information terms vanish. Under these conditions, Theorem~\ref{th:xu_like_bound_mi} simplifies to the classical bound of Theorem~\ref{th:raginsky_bound}, recovering earlier results on generalization in the absence of data augmentation. This demonstrates that our bound strictly generalizes prior information-theoretic results by capturing the added complexity and structure introduced by augmentation schemes, thereby offering a more flexible and realistic framework for modern training procedures that rely heavily on data augmentation. 

\textbf{Remark 3 (full augmentation setting).} Consider the case where the set of augmentations $E = \mathcal{G}$, meaning that the learning algorithm has access to the full transformation group during training. In this setting, the third term in the generalization bound established in Theorem~\ref{th:xu_like_bound_mi} vanishes. This is because the mutual information $I_{\ell}^{Z_i}(E;W) = 0$ for all $i$, as $W$ becomes conditionally independent of $\mathcal{G}$ given $Z_i$ when the full group is deterministically and uniformally applied to each example. This simplification highlights a key benefit of using structured and complete augmentation schemes in some cases: they can reduce the dependency of the model on the specific augmented instances, thereby tightening the generalization bound.

\subsection{A tighter bound via individual sample and augmentation mutual information }
\label{sec:bu_bound_augment_empir}
Theorem \ref{th:xu_like_bound_mi} provides a generalization bound that characterizes the effect of data augmentation through global information-theoretic quantities, notably the mutual information between the entire training dataset and the learned model. This global perspective does not distinguish between samples that have a large impact on the model and those that contribute very little. In practice, however, modern learning algorithms often exhibit non-uniform sensitivity to training data, that is, the model may rely heavily on just a small subset of influential examples or challenging augmentations, while being largely unaffected by the rest. When this is the case, using a single mutual information quantity over the whole dataset can obscure this variability, resulting in a potentially loose bound. To address this limitation, we leverage per-sample technique introduced by \citep{bu2020tightening} and provide Theorem \ref{th:bu_like_bound_mi}, which refines the analysis by decomposing the generalization gap into per-sample and per-augmentation mutual information terms. This finer-grained characterization allows for a tighter and more interpretable bound, offering a clearer understanding of how individual data points and transformations influence generalization.

\begin{theorem}
\label{th:bu_like_bound_mi}
Suppose assumptions \ref{assumption_1} and \ref{assumption_2} are satisfied, then  
$$ 
\lvert \mathrm{gen}(\mathcal{D}, P_{W|S,E}) \rvert \leq R \left( \underbrace{\sqrt{2D_{\mathrm{KL}} (\mathcal{D}\Vert  \mathcal{D}_\mathcal{G} \circ \mathcal{D})}}_\text{ Distribution shift} + \underbrace{\frac{1}{m} \sum_{i=1}^m\sqrt{2I_{\ell_\mathcal{G}}(Z_i;W)}}_\text{ MI under oribt loss} + \underbrace{\frac{1}{mn} \underset{S \sim \mathcal{D}^m}{\mathbb{E}} \left[ \sum_{i=1}^{m} \sum_{j=1}^{n} \sqrt{2I_\ell^{Z_i}(G_j;W)}\right]}_\text{ Augmentation MI} \right)\, .
$$
\end{theorem}

The proof of Theorem \ref{th:bu_like_bound_mi} is provided in Appendix \ref{proof:bu_like_bound_mi}.

\textbf{Remark 1 (Comparison with the bound in Theorem \ref{th:xu_like_bound_mi}.)} Theorem \ref{th:bu_like_bound_mi} offers a refined information-theoretic lens on generalization in augmented learning settings. Unlike the bound estalished in Theorem \ref{th:xu_like_bound_mi}  that incorporates the mutual information between the entire training dataset, entire transformations and the model, Theorem \ref{th:bu_like_bound_mi} decomposes the generalization gap into fine-grained contributions arising from individual training examples and their associated augmentations. This per-sample and per-augmentation decomposition is particularly insightful in the modern regime of overparameterized models and large-scale data, where empirical observations suggest that models often exhibit non-uniform dependence on the training data \citep{toneva2018empirical, tirumala2023d4, mindermann2022prioritized}.

In practice, only a subset of training examples, often those that are semantically ambiguous or close to decision boundaries, tend to be unforgettable by deep models, while many other examples are learned in a more redundant manner. Similarly, certain transformations may significantly affect model predictions, whereas others leave the model's output largely unchanged. Theorem \ref{th:bu_like_bound_mi} captures this asymmetry explicitly through the per-example and per-augmentation mutual information terms $I_{\ell_\mathcal{G}}(Z_i;W)$ and $I_\ell^{Z_i}(G_j;W)$, quantifying how much each individual data point and its augmented versions influence the learned model.

This refined analysis has both theoretical and practical implications. From a theoretical standpoint, it provides tighter generalization bound. Using the chain rule of mutual information and Jensen’s inequality, we can show that $1/m \sum_{i=1}^m\sqrt{2I_{\ell_\mathcal{G}}(Z_i;W)}$ and $1/n\sum_{j=1}^{n} \sqrt{2I_\ell^{Z_i}(G_j;W)}$ are tighter than their global counterparts. From a practical perspective, it enables targeted augmentation strategies: for example, one might prioritize stronger or more diverse augmentations for samples with high mutual information, as these are the ones the model is most sensitive to. This aligns with empirical strategies like curriculum learning \citep{soviany2022curriculum, zhang2025fastdinov2} and sample reweighting \citep{zhou2022model}, which implicitly aim to balance the contribution of difficult examples during training.

\textbf{Remark 2 (Implications for invariance learning.)} A central goal of data augmentation is to induce invariance in the learned model, that is, to ensure that transformations from a group $\mathcal{G}$ do not significantly affect model predictions. Theorem \ref{th:bu_like_bound_mi} gives a formal handle on this idea. In particular, the term $I_\ell^{Z_i}(G_j;W)$ quantifies how much information about the specific augmentation $G_j$ applied to the specific sample $Z_i$ is retained in the learned model $W$. A small value for all $j$ of this term indicates that the model exhibits a high degree of invariance to the transformations, i.e., it learns features that are stable under the group action. Conversely, a large value signals that the model is still sensitive to these variations, failing to fully internalize the intended invariances. This can guide adaptive augmentation strategies: if some samples $Z_i$ exhibit large $I_\ell^{Z_i}(G_j;W)$, we may want to apply stronger, targeted, or diverse augmentations to those examples to enforce invariance more effectively.

\subsection{Applications}

In this subsection, we illustrate the implications of our bounds through two representative applications. The first considers an analytically tractable Gaussian mean estimation problem, where all information-theoretic terms can be computed in closed form, allowing us to explicitly visualize the trade-offs induced by data augmentation. The second focuses on a finite hypothesis and finite augmentation setting, where the mutual information terms admit simple cardinality based  bounds, yielding interpretable guarantees.

\subsubsection{Gaussian mean estimation} 
To illustrate the central message of our theoretical analysis; that the geometry of augmentations induces a quantifiable trade-off between distributional shift and information contraction, we consider a controlled Gaussian setting inspired by \cite{nadjahi2024slicing} and \cite{bu2020tightening} . This analytically tractable setting allows all quantities in Theorem \ref{th:bu_like_bound_mi} to be computed in closed forms, providing direct insight into how the geometry of augmentation affects the generalization bound. We consider a base dataset $Z_1, ..., Z_m \overset{i.i.d.}{\sim} \mathcal{N}(\mu, s^2I_d)$. Each sample $Z_i$ is augmented $n$ times through independent perturbations $G_1, ..., G_n \overset{i.i.d.}{\sim} \mathcal{N}(0, t^2I_d)$. The learner then aggregates information from all augmented views to produce a stochastic hypothesis
$$W = \frac{1}{mn} \sum_{i=1}^{m} \sum_{j=1}^{n}  (Z_i + G_j) + \epsilon, \quad \epsilon \sim \mathcal{N}(0, \nu^2I_d)$$
where $\epsilon$ models algorithmic randomness or implicit regularization effects (e.g., from stochastic optimization). The learning objective is based on the clipped squared loss $\ell_M(w,z) = \min\{||w-z||^2, M\}$, where $M>0$ is a clipping constant controlling the maximum penalty. This choice ensures that the loss remains bounded, hence, sub-Gaussian with parameter $R= M/2$, while preserving the interpretability of the standard squared loss in the regime where $||w-z||^2 \ll M$. In the Gaussian mean estimation setting, this loss provides a smooth approximation to the true quadratic objective while guaranteeing the theoretical assumptions required for the generalization bound in Theorem \ref{th:bu_like_bound_mi} to hold. This formulation provides a simplified analytical environment for examining how the interplay between data variance $s^2$, augmentation strength $t^2$, and training stochasticity $\nu^2$ shapes the generalization behavior through the lens of information-theoretic bounds. 

This setting admits a closed-form expression of the generalization bound, summarized in the following result.

\begin{corollary}
\label{cor:gaussian_mn_bound}
Consider the setting introduced above. Then,
\begin{align*}
 \lvert \mathrm{gen}(\mathcal{D}, P_{W|S,E}) \rvert &\leq R \sqrt{d \left( \frac{s^2}{s^2+t^2} - 1 - \log \frac{s^2}{s^2+t^2}  \right) } + R  \sqrt{d \log \frac{m(s^2+t^2 /mn) + m^2\nu^2}{(m-1)s^2 + t^2/n + m^2\nu^2} } \\
 &+ R \sqrt{d \log \left(  1 + \frac{t^2/n^2}{(n-1)/n^2t^2 + (m-1)/m^2s^2 + \nu^2} \right)}.   
\end{align*}
\end{corollary}
The proof of Corollary \ref{cor:gaussian_mn_bound} is provided in Appendix \ref{proof:gaussian_mn_bound}.

This corollary gives an explicit characterization of the expected generalization gap in the Gaussian mean model with Gaussian augmentation noise. Each term corresponds to a distinct mechanism by which information about the training set can influence the learned model. The first term is the distribution shift contribution (KL-divergence) and increases with augmentation strength $t^2$: larger augmentations move the training distribution away from the test distribution. The second term measures the dependence between model parameters and training samples (the per-sample mutual information contribution); it is attenuated by larger sample size $m$, by algorithmic noise $\nu^2$, and decreases as $t^2$ increases (since augmentations inflate the learner variance and thereby reduce per-sample correlation). The third term captures augmentation sensitivity (how much a particular augmented view affects the learned model); it increases with $t^2$ but decreases with more augmentations $n$ or larger algorithmic noise $\nu^2$, because any single augmentation then has less influence on the aggregate. Overall, the bound quantifies the trade-off: increasing augmentation strength reduces per-sample information (a regularizing effect) but increases distributional mismatch and augmentation-induced variability, and the optimal augmentation regime balances these opposing influences. Figure \ref{fig:gaussian_mn_bound} illustrates the behavior of the generalization bound derived in Corollary \ref{cor:gaussian_mn_bound} as a function of the augmentation variance $t^2$ and the augmentation multiplicity $n$
. The plots clearly reflect the theoretical trade-offs predicted by the bound.

In the limit $t^2 \to 0$, the KL-divergence term and the augmentation mutual information both vanish, meaning that augmentations become negligible and the generalization gap is dominated by the dependence between the model parameters and individual training samples. Conversely, as $t^2 \to \infty$, the KL-divergence term grows logarithmically without bound, reflecting excessive distortion between the augmented and original data distributions. Meanwhile, the per-sample information term tends to zero, and the augmentation mutual information saturates to a small positive constant, approximately $\sqrt{d\log \left( n/(n-1) \right) }$.

In general, obtaining closed-form expressions for the three information-theoretic terms appearing in the generalization bound in Theorem~\ref{th:bu_like_bound_mi} is analytically intractable for realistic datasets and more general transformation groups. This intractability stems from the fact that these quantities depend on the joint distribution of data, model parameters, and augmentations, which are defined implicitly by the stochastic training dynamics of deep models. For simple or discrete settings, however, one can sometimes obtain explicit or bounded expressions.

\begin{figure}[htbp]
  \centering
  \begin{subfigure}[b]{1.0\linewidth}
    \centering
    \includegraphics[width=\linewidth]{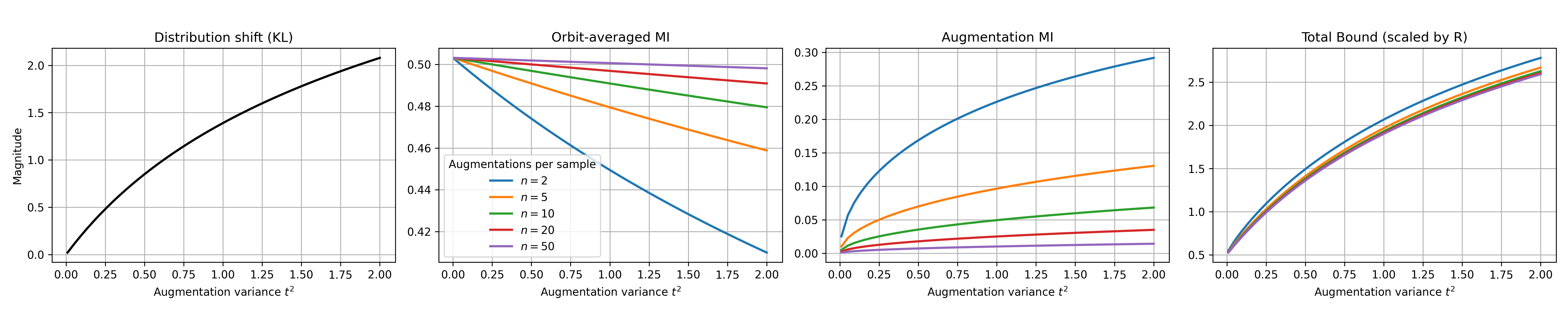}
    \label{fig:toy_gaussian_bound_components_vs_Augmentation_Strength}
  \end{subfigure}

  \begin{subfigure}[b]{1.0\linewidth}
    \centering
    \includegraphics[width=\linewidth]{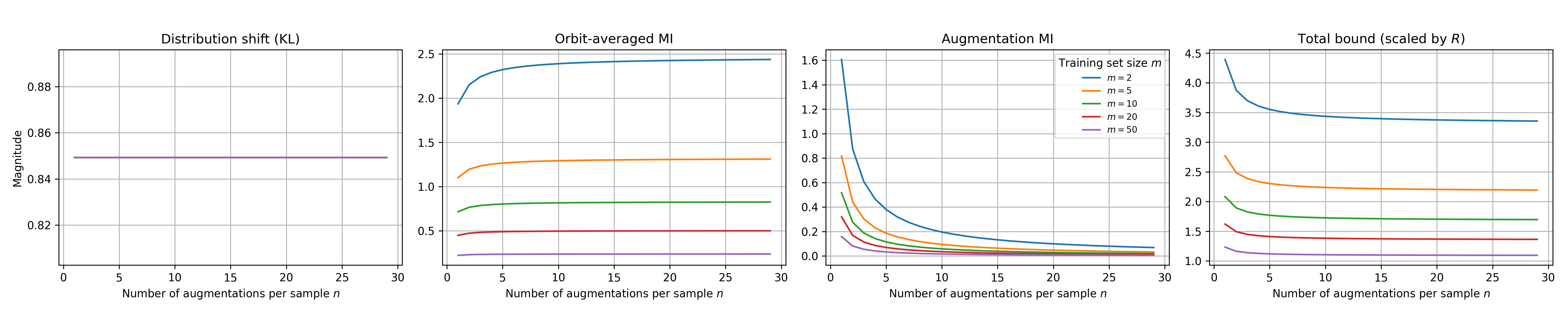}
    \label{fig:toy_gaussian_bound_components_vs_Number_Augmentations}
  \end{subfigure}

    \caption{Information-theoretic generalization bound components under varying augmentation conditions.
Top row: Evolution of the distribution shift (KL-divergence), orbit-averaged mutual information, augmentation mutual information, and total bound (scaled by $R$) as a function of augmentation variance $t^2$. Increasing augmentation strength amplifies distribution shift and total bound while reducing orbit-averaged mutual information. Bottom row: Dependence of the same four components on the number of augmentations per sample $n$, for different training set sizes $m$. The augmentation mutual information term decays rapidly with $n$, leading to a tighter overall bound for larger $n$ and m.}
  \label{fig:gaussian_mn_bound}
\end{figure}

\subsubsection{Finite hypothesis and augmentation setting} 
Theorem~\ref{th:xu_like_bound_mi} provides a general mutual information based bound on the expected generalization gap, applicable to arbitrary hypothesis spaces $\mathcal{W}$ and augmentation groups $\mathcal{G}$. However, in many practical scenarios, both the hypothesis space and the augmentation set are finite: for example, $\mathcal{W}$ may consist of a discrete set of candidate models (e.g., from a model selection procedure), and $\mathcal{G}$ may correspond to a fixed finite set of augmentation transformations (e.g., permutation group). In this setting, the mutual information terms in Theorem~\ref{th:xu_like_bound_mi} admit simple cardinality-based upper bounds
$$ I_{\ell_\mathcal{G}}(S;W) \le H(W) \leq \log |\mathcal{W}|, \quad \text{and} \quad  I_\ell^{Z_i}(E;W) \le H(E) \leq \log |\mathcal{G}|.$$
Substituting these into Theorem~\ref{th:xu_like_bound_mi} yields an explicit bound where the complexity of the learning algorithm and the augmentation scheme is expressed purely in terms of $|\mathcal{W}|$ and $|\mathcal{G}|$ leading to the following corollary 

\begin{corollary}
\label{cor:finite_W_G}
Suppose that the hypothesis space $\mathcal{W}$ is finite with $|\mathcal{W}| < \infty$  
and the augmentation group $\mathcal{G}$ is finite with $|\mathcal{G}| < \infty$.  
Then, for any learning algorithm $P_{W \mid S,E}$, we have
$$
\big| \mathrm{gen}(\mathcal{D}, P_{W \mid S,E}) \big| \le
R \left(
\sqrt{2\, D_{\mathrm{KL}}\!\left( \mathcal{D} \,\Vert\, \mathcal{D}_\mathcal{G} \circ \mathcal{D} \right)}
+ \sqrt{\frac{2}{m} \log |\mathcal{W}|}
+ \sqrt{\frac{2}{n} \log |\mathcal{G}|}
\right).$$
\end{corollary}

A key feature of Corollary~\ref{cor:finite_W_G} is its direct dependence on the cardinalities $|\mathcal{W}|$ and $|\mathcal{G}|$. In contrast to the general case, where the mutual information terms in Theorem~\ref{th:xu_like_bound_mi} must be estimated or bounded using analytical techniques, the finiteness of both spaces yields immediate bounds expressed in closed form through 
$\log |\mathcal{W}|$ and $\log |\mathcal{G}|$. Importantly, these complexity contributions are accompanied by favorable normalization factors that reflect the averaging effects of both data and augmentation. The contribution associated with the hypothesis space scales as $\log |\mathcal{W}| /m$, indicating that the influence of hypothesis complexity is diluted as the number of training samples increases. This mirrors classical finite class generalization results \citep{shalev2014understanding}, where increasing data availability progressively suppresses dependence on the size of the hypothesis space. Similarly, the augmentation-related contribution scales as $\log |\mathcal{G}| / n$, showing that the sensitivity of the learned model to individual transformations diminishes as more augmented views are aggregated for each training example. As $m$ and $n$ become large, these terms decay to zero, capturing the regularizing role of data and augmentation through averaging over samples and transformations, respectively.

However, this simplification comes with limitations: if either $|\mathcal{W}|$ or $|\mathcal{G}|$ is large, the bound can become loose and fail to reflect the effective complexity of the problem. Moreover, the result is inherently worst-case, making no attempt to exploit possible structure such as redundancy among hypotheses in  $|\mathcal{W}|$ or statistical dependencies between augmentations in $|\mathcal{G}|$ that could lead to significantly tighter bounds in practice. Also, the first term is entirely independent of $|\mathcal{W}|$ and $|\mathcal{G}|$ and cannot, in general, be bounded using only their cardinalities. Its value depends on how the augmentations alter the data distribution, and bounding it requires a measure of the geometric effect of the group action on the input space.

To address these limitations, especially in settings where $\mathcal{G}$ is large or even continuous, it is natural to seek alternative complexity measures that go beyond mere cardinality and better capture the effective impact of augmentations on the data distribution. Rather than counting the number of transformations, we can quantify their magnitude of effect on the data space, leading to bounds that adapt to the actual strength and diversity of the augmentations rather than their raw quantity.

In the next section, we introduce such a measure and show how it can be used to bound all three terms in Theorem~\ref{th:bu_like_bound_mi} in a way that remains meaningful even for infinite or continuous groups.

\section{Bounding generalization terms via group diameter}
In the presence of data augmentation modeled by a transformation group $\mathcal{G}$, the geometry of the group action plays a fundamental role in shaping the generalization behavior of learning algorithms.  We introduce the group diameter $\Delta_\mathcal{G}$ as an intrinsic geometric quantity that captures the maximal average deviation a point can undergo under random transformations from the group. Formally, we define  $\Delta_\mathcal{G}$ as
\begin{equation}
\label{eq:def_delta}
\Delta_\mathcal{G} \eqdef \underset{z \in \mathcal{Z} }{\sup} \underset{G \sim \mathcal{D}_\mathcal{G} }{\mathbb{E}} [d_\mathcal{Z}(z,Gz)],     
\end{equation}
where $d_\mathcal{Z}$ denotes a metric on the input space $\mathcal{Z}$, and $\mathcal{D}_\mathcal{G}$ is a probability measure over $\mathcal{G}$, typically the normalized Haar measure. This diameter reflects the worst-case average geometric displacement of a point under the group action, and thus serves as a proxy for the size or spread of its orbit. Intuitively, a small $\Delta_\mathcal{G}$ corresponds to transformations that induce minor, often task-irrelevant transformations (e.g., small rotations or translations), while a large diameter indicates stronger transformations that may alter semantically meaningful features of the data. As such, $\Delta_\mathcal{G}$ provides a principled measure of how augmentation affects both the geometry of the input space and the generalization properties of the learning algorithm. Our analysis focuses on the generalization bound in Theorem~\ref{th:bu_like_bound_mi}, but the insights also extend naturally to Theorem~\ref{th:xu_like_bound_mi}. Theorem~\ref{th:bu_like_bound_mi} expresses the generalization gap in terms of three mutual information quantities: (1) the KL-divergence between the data distribution and its augmented counterpart, (2) the orbit-averaged mutual information between training examples and the learned model, and (3) the augmentation mutual information, capturing the sensitivity of the learned parameters to the applied transformations.

In this section, we present a unified framework that connects all three components of Theorem~\ref{th:bu_like_bound_mi} to the group diameter $\Delta_\mathcal{G}$. This approach relies on interpreting $\Delta_\mathcal{G}$ as a surrogate for the strength or reach of the augmentations. The central insight is that when $\Delta_\mathcal{G}$ is small, the group acts in a nearly identity preserving way on the data space, resulting in limited statistical shift, reduced informativeness of the augmentation, and bounded sensitivity of the learning algorithm to transformations. Under suitable smoothness assumptions (e.g., Lipschitz continuity of the data density and stability of the learning algorithm), both the distribution-shift and the augmentation mutual-information terms can be tightly bounded explicitly as an increasing function of $\Delta_\mathcal{G}$. 

For the orbit-averaged mutual information term, we do not provide an explicit closed-form dependence on $\Delta_\mathcal{G}$. Nevertheless, Proposition~\ref{prop:diff_mi} ensures that orbit averaging always contracts information, with the degree of contraction determined by the expected KL-divergence between the conditional posteriors $P_{W|Z,G}$ and their orbit-averaged counterpart $P^{\mathcal{G}}_{W|Z}$ defined in Equation \ref{eq:orbit_qveraged_output}. This expected divergence reflects the variability of the learner’s output across augmented views and thus depends implicitly on $\Delta_\mathcal{G}$. 

\subsection{A geometric bound on the distribution shift Term}
A key quantity appearing in the generalization bound established in Theorem~\ref{th:bu_like_bound_mi} is the KL-divergence $D_{\mathrm{KL}} (\mathcal{D}\Vert  \mathcal{D}_\mathcal{G} \circ \mathcal{D})$, which quantifies the shift between the original and augmented distributions. To control this divergence, we relate it to the geometry of the group action via the group diameter $\Delta_\mathcal{G}$. Intuitively, when the transformations in $\mathcal{G}$ are small or act smoothly on the data space, the induced augmentation distribution remains close to the original one, leading to a small divergence. Formally, if the data distribution $\mathcal{D}$ admits a density $p$ that is sufficiently regular within each orbit, that is, Lipschitz continuous, and the distribution $\mathcal{D}_\mathcal{G} \circ\mathcal{D} $ admits a density $p_{\mathcal{G}}$ that is uniformly bounded away from zero, the KL-divergence can be explicitly bounded in terms of $\Delta_\mathcal{G}$. The following proposition formalizes this connection and demonstrates how geometric regularity of the group action translates into probabilistic stability under augmentation

\begin{proposition}
\label{prop:bound_first_term}
Let $\mathcal{D}$ be a distribution with a continuous density $p$ that is $L_p$-Lipschitz within each orbit. Let $p_{\mathcal{G}}(z) \eqdef \underset{G \sim \mathcal{D}_\mathcal{G}}{\mathbb{E}}[p(Gz)]$ be the density of the distribution $\mathcal{D}_\mathcal{G} \circ \mathcal{D}$. Assume that $p_{\mathcal{G}}$ is uniformly lower bounded by  $p_{\mathcal{G}}\geq c >0 $. Then
$$D_{\mathrm{KL}} (\mathcal{D}\Vert  \mathcal{D}_\mathcal{G} \circ \mathcal{D}) \leq \frac{L_p \Delta_\mathcal{G}}{c}.$$
\end{proposition}
The proof of Proposition \ref{prop:bound_first_term} is provided in Appendix \ref{proof:bound_first_term}.

This result provides a precise quantitative link between the geometric impact of data augmentations and the statistical stability of the data distribution under these augmentations. Specifically, it shows that the KL-divergence $D_{\mathrm{KL}} (\mathcal{D}\Vert  \mathcal{D}_\mathcal{G} \circ \mathcal{D})$ is jointly governed by three quantities: the Lipschitz constant $L_p$, which measures the smoothness of the data density; the lower bound $c$, which ensures that the augmented distribution retains sufficient mass where the original distribution is supported; and the group diameter $\Delta_\mathcal{G}$, which captures the average geometric displacement induced by the transformations. Intuitively, when the group acts almost as the identity, the resulting augmented distribution remains close to the original one, yielding a small KL-divergence. Conversely, if either the density is highly sensitive or the transformations are too aggressive, the augmented distribution may diverge substantially from the original, potentially harming generalization. This bound therefore provides theoretical justification for favoring smooth data distributions and low-diameter augmentation groups in practice.

A particularly important case arises when the data distribution $\mathcal{D}$ is $\mathcal{G}$-invariant under the action of the augmentation group $\mathcal{G}$, and the transformations are drawn uniformly from $\mathcal{G}$. In this setting, applying random augmentations does not alter the distribution of the data: the augmented distribution coincides exactly with the original one, i.e., $\mathcal{D}_\mathcal{G} \circ \mathcal{D} = \mathcal{D}$. As a result, the KL-divergence between them vanishes
$$D_{\mathrm{KL}} (\mathcal{D}\Vert  \mathcal{D}_\mathcal{G} \circ \mathcal{D}) = 0.$$
This behavior is consistent with the bound in Proposition~\ref{prop:bound_first_term}, which depends on the Lipschitz constant $L_p$ of the data density. If $\mathcal{D}$ is $\mathcal{G}$-invariant, then its density $p(z)$ must remain constant along orbits, i.e., it is invariant under the group action, so the directional derivatives of $p$ along group directions vanish, implying $L_p=0$. Consequently, the KL-divergence bound becomes zero.

Settings where the data distribution $\mathcal{D}$ is exactly $\mathcal{G}$-invariant and augmentations are sampled uniformly over the group $\mathcal{G}$ are widely adopted in the literature for understanding the benefits of data augmentation and invariance learning \citep{chen2020group, lyle2020benefits, shen2022data}. Such assumptions formalize the intuition that transformations preserving all task-relevant structure should leave the underlying data distribution unchanged, ensuring that augmentation acts purely as a regularizer without introducing distribution shift.

Our framework is more general and does not rely on such strong symmetry assumptions. We allow for approximate  $\mathcal{G}$-invariance and potentially non-uniform augmentation schemes, reflecting more realistic training pipelines where the data distribution may exhibit small but non-negligible variations along group orbits and augmentation policies may be biased toward certain transformations. In these more general settings, the KL-divergence term need not vanish. Nevertheless, Proposition~\ref{prop:bound_first_term} ensures that it can be controlled in terms of two interpretable geometric quantities: the Lipschitz constant $L_p$ of the density along group orbits and the group diameter $\Delta_\mathcal{G}$. This provides a direct link between the geometric reach of the augmentations and the statistical shift they induce in the data distribution, offering theoretical guidance for balancing augmentation strength with distributional stability.

This bound can be directly applied in experimental settings to tune augmentation parameters (e.g., maximum rotation angle, translation distance, or color jitter intensity). By estimating or constraining $\Delta_\mathcal{G}$, practitioners can select augmentation strengths that are strong enough to improve invariance while keeping the induced distribution shift within a theoretically justified range, thus avoiding over augmentation that could degrade generalization.

\subsection{Mutual information contraction under orbit-averaging}

In the mutual information framework, the expected generalization gap is bounded in terms of how much information the algorithm memorizes from the training data. So higher mutual information means potentially worse generalization. The loss function affects the generalization bound in two ways: through the sub-Gaussian constant $R$ and through the mutual information term $I(S;W)$. The constant $R$  captures the tail behavior of the loss. A loss function with heavier tails or higher variance leads to a larger $R$, thereby loosening the bound. For the generalization bound to be tight, the loss function must be sub-Gaussian with a small constant $R$. The mutual information $I(S; W)$ is a purely distributional quantity between the training data set $S$ and the output $W$ of the learning algorithm and its formula does not include the loss function explicitly. However, the learning algorithm, which maps $S \mapsto W$, is defined using the loss function. So even though the mutual information formula does not contain $\ell$, the distribution of $W$ given $S$ does depend on $\ell$. Intuitively, if you change the loss function, the way the algorithm processes the data changes. This, in turn, changes the conditional distribution $P_{W|S}$ , which directly affects $I(S; W)$. In other words, the loss function $\ell$ influences indirectly the mutual information $I(S; W)$ through the behavior of the learning algorithm. In particular, sharper loss functions, those with steeper gradients and higher curvature near the minimum (e.g., cross-entropy or hinge loss), penalize deviations from correct predictions more strongly. As a result, the learning algorithm must fit the training data more precisely to minimize such a loss, effectively encoding more specific information from the dataset into the learned parameters. This tighter coupling between the model and the data increases the mutual information $I(S; W)$, indicating that more information about the dataset has leaked into the model. In contrast, flatter loss functions (e.g., Mean Squared Error or smoothed hinge loss \citep{luo2021learning}), which provide more tolerance around the minimum, encourage broader generalization and introduce less sensitivity to specific data points. This leads to a lower mutual information.

To formally capture the impact of the orbit-averaged loss function on generalization, we compare two mutual information quantities: the original loss induced mutual information $I_{\ell}(Z_i;W)$ and its orbit-averaged counterpart $I_{\ell_\mathcal{G}}(Z_i;W)$. The first measures the dependence between the data and the learned model without considering group structure, while the second accounts for invariance by replacing the standard conditional distribution $P_{W \mid Z}$ with its orbit-averaged version 
\begin{equation}
\label{eq:orbit_qveraged_output}
P^{\mathcal{G}}_{W \mid Z} \eqdef \underset{G \sim \mathcal{D}_\mathcal{G} }{\mathbb{E}}[P_{W \mid Z,G}].
\end{equation}
This orbit-averaging smooths the conditional distribution of $W$ over the group orbit of each sample, effectively reducing sensitivity to specific data instantiations within the orbit.

Under the assumption that the data distribution is $\mathcal{G}$-invariant, the difference between these two mutual information terms can be written as
\begin{proposition}
\label{prop:diff_mi}
Suppose that the original data distribution $\mathcal{D}$ is $\mathcal{G}$-invariant. Then, For every algorithm and group $\mathcal{G}$,
$$I_{\ell}(Z;W) - I_{\ell_\mathcal{G}}(Z;W) = \underset{Z \sim \mathcal{D} }{\mathbb{E}} \underset{G \sim \mathcal{D}_\mathcal{G} }{\mathbb{E}} [D_{KL}(P_{W \mid Z,G} || P^{\mathcal{G}}_{W \mid Z})] \geq 0.$$
Moreover, $I_{\ell}(Z;W) - I_{\ell_\mathcal{G}}(Z;W) = 0$ holds iff each $P_{W \mid Z,G}$ equals the averaged conditional almost surely.
\end{proposition}
The proof of Proposition \ref{prop:diff_mi} is provided in Appendix \ref{proof:diff_mi}.

This proposition reveals a fundamental property of orbit averaging: it can only reduce the mutual information of learning algorithms. In other words, the statistical dependence between the data and the learned model is never increased by averaging the loss over group actions. Moreover, the reduction can be expressed exactly as the expected KL-divergence between the original conditional distribution $P_{W \mid Z,G}$ and its orbit averaged counterpart $P^{\mathcal{G}}_{W \mid Z}$ defined in Equation \ref{eq:orbit_qveraged_output}. This term measures how much the conditional distribution of the learned model $W$ changes when the input is replaced by a specific group transform $gZ$, compared to when we forget which augmentation was applied and only know the orbit. In particular, a small value indicates that the learner produces nearly identical output distributions for all augmented versions of a given input, reflecting a high degree of $\mathcal{G}$-invariance. In contrast, a large value reveals substantial variation across the orbit, signaling weaker invariance.

The magnitude of the reduction in mutual-information in Proposition \ref{prop:diff_mi} can be related to the geometry of the group action through the diameter of the group $\Delta_\mathcal{G}$ that measures the average displacement of a point under the augmentation distribution. Crucially, a small $\Delta_\mathcal{G}$ does not mean every group element produces only a tiny perturbation; rather it means that on average, under the augmentation distribution, transformed points lie close to the original sample, or the orbit mass is concentrated in a small neighborhood of the sample. In that regime, the orbit-averaged loss $\ell_\mathcal{G}$ is a mild smoothing of $\ell$: for most sampled augmentations $G$ the conditional distribution $P_{W \mid Z,G}$ will be close (e.g., in KL-divergence) to the orbit-averaged posterior $P^{\mathcal{G}}_{W \mid Z}$. Hence, the expected KL-divergence in Proposition \ref{prop:diff_mi} is small, and one obtains a correspondingly small information contraction $I_{\ell}(Z;W) - I_{\ell_\mathcal{G}}(Z;W) \approx 0$. 

By contrast, when the group diameter $\Delta_\mathcal{G}$ is large, the augmentation distribution assigns substantial probability mass to transformations that move each input $z$ far from its original position. In this regime, orbit averaging effectively blends together losses computed at widely separated points in the input space. Consequently, the orbit-averaged posterior $P^{\mathcal{G}}_{W \mid Z}$ becomes a mixture of conditionals $P_{W \mid Z,G}$ that correspond to statistically distinct data configurations. From an information–theoretic perspective, this mixture operation acts as a strong contraction of the $Z \mapsto GZ \mapsto W$ channel: by the convexity of the KL-divergence and the Strong Data Processing Inequality \citep{dobrushin1956central, cover1999elements}, the expected divergence 
$$\underset{Z \sim \mathcal{D} }{\mathbb{E}} \underset{G \sim \mathcal{D}_\mathcal{G} }{\mathbb{E}} [D_{KL}(P_{W \mid Z,G} || P^{\mathcal{G}}_{W \mid Z})]$$ tends to increase with the spread of the orbit, as larger $\Delta_\mathcal{G}$ typically induces greater variability among the conditionals $P_{W|Z,G}$. Intuitively, the more heterogeneous the augmented views of a sample are, the more the learner must average over or discard augmentation-specific information in order to remain consistent across the orbit. This loss of augmentation-specific information manifests as a reduction of $I_{\ell_\mathcal{G}}(Z;W)$ relative to $I_{\ell}(Z;W)$. Geometrically, when orbits corresponding to distinct samples begin to overlap, a phenomenon amplified by aggressive data augmentation induced by large $\Delta_\mathcal{G}$ \citep{zhang2025augmentation}, the algorithm’s output distribution $P_{W \mid Z}$ becomes increasingly insensitive to fine variations in $Z$. As a result, the dependence between the learned hypothesis $W$ and individual training samples $Z$ is reduced, leading to a smaller orbit-averaged mutual information.

\subsection{Bounding augmentation sensitivity via reverse Pinsker inequality}

The third term in Theorem~\ref{th:bu_like_bound_mi} captures the model's sensitivity to the randomness introduced by data augmentation. Even for a fixed data orbit, the specific group element $G_j$ applied during training can affect the training dynamics of a model and therefore carry additional information to the learned hypothesis $W$. We identify conditions under which this augmentation-sensitive mutual information term becomes small, leading to a simplified generalization bound. To this end, we derive an upper bound on the augmentation mutual information under structural assumptions on the augmentation group $\mathcal{G}$ and the learning algorithm. Specifically, assuming smooth dependence of $P_{W|z,g}$ on $g$ in total variation distance, we obtain a bound in terms of the group diameter $\Delta_\mathcal{G}$. This result formalizes the intuition that smoother group actions and models that are more invariant to augmentation lead to tighter generalization guarantees.

In this section, the total variation distance is employed to quantify the smoothness of the learning algorithm with respect to the group action. Specifically, we assess the sensitivity of the conditional distribution $P_{W|z,g}$ to variations in the group element $g$, for a fixed input $z \in \mathcal{Z}$. The total variation distance between two conditional distributions $P_{W \mid z,g}$ and $P_{W|z,g'}$ is defined as
$$\mathrm{TV}(P_{W|z,g}, P_{W|z,g'}) \eqdef \underset{A \subseteq \mathcal{W} }{\sup} |P_{W|z,g}(A) - P_{W|z,g'}(A)|.$$
This distance captures the worst-case difference in probability mass assigned to measurable sets, and as such, it offers a robust and interpretable measure of sensitivity and admits clean information-theoretic inequalities such as Pinsker’s and reverse Pinsker's inequalities \citep{polyanskiy2014lecture}. That said, other divergences such as the KL-divergence and the Wasserstein distance could also be used to measure sensitivity.

\subsubsection{Finite and discrete case analysis}
In this subsection, we restrict ourselves to the case of discrete spaces, where the data domain $\mathcal{Z}$, the augmentation group $\mathcal{G}$, and the hypothesis space $\mathcal{W}$ are all finite. This setting allows us to leverage a reverse Pinsker inequality \citep{polyanskiy2014lecture, sason2015reverse}, which provides an upper bound on the KL-divergence in terms of the square of the total variation distance. Specifically, for two probability distributions $P$ and $Q$ defined on a finite set $\mathcal{W}$, we have
\begin{equation}
\label{eq:reverse_pinsker}
D_{KL}(P||Q) \leq \frac{\mathrm{TV}(P, Q)^2 }{Q_{min}}, 
\qquad
Q_{min} \eqdef \underset{w \in \mathcal{W}}{\min}\,Q(w)>0.
\end{equation}

Inequality \eqref{eq:reverse_pinsker} shows that, in discrete settings, KL-divergence can be controlled in terms of the total variation distance provided that the distribution $Q$ does not assign arbitrarily small probability mass to elements of its support. Intuitively, if no point in the support of $Q$ has arbitrarily small mass, then the relative entropy cannot explode too much, even if $P$ deviates from $Q$. The worst-case explosion occurs near points where $Q(w)$ is very small, which are precisely where the divergence may blow up. In contrast, for continuous spaces, an upper bound on the KL-divergence in terms of total variation distance does not hold in general. That is, there exist pairs of continuous distributions $P$ and $Q$ on the same space such that $\mathrm{TV}(P, Q)$ is arbitrarily small, yet $D_{KL}(P||Q)$ is arbitrarily large or even infinite. Therefore, in the continuous case, reverse Pinsker bound becomes more delicate and requires more stronger regularity assumptions.

Although inequality \eqref{eq:reverse_pinsker} is formally valid in any finite setting, its usefulness crucially depends on the magnitude of $Q_{min}$. In particular, when the cardinality of $\mathcal{W}$, denoted by $|\mathcal{W}|$, grows exponentially with the number of parameters $d = \dim(\mathcal W)$, the minimum mass $Q_{\min}$ necessarily scales at most as $1/|\mathcal{W}|$. Even in the most favorable case where $Q$ is uniform over $\mathcal{W}$, one has $Q_{min}= 1/|\mathcal{W}|$, and any deviation from uniformity further decreases this minimum mass. Consequently, the bound in \eqref{eq:reverse_pinsker} becomes exponentially loose, unless $\mathrm{TV}(P, Q)$ is exponentially small in $d$. 

Nevertheless, this bound can remain meaningful in restricted regimes where the analysis is carried out on an effective hypothesis space $\mathcal{W}_{eff} \subset \mathcal{W}$  on which $Q$ is supported, and where the minimum mass $Q_{min}$ is defined relative to this restricted space. Such a restriction may arise when $Q$ is strongly concentrated on a low-complexity subset, through explicit support truncation, posterior concentration, or compression-based representations of hypotheses, provided that $|\mathcal{W}_{eff}|$ grows sub-exponentially with the dimension $d$. In the absence of such a restriction on the effective hypothesis space, controlling KL-divergence via total variation in high-dimensional discrete hypothesis spaces is fundamentally infeasible.

This control, given by Equation \ref{eq:reverse_pinsker}, is particularly useful in our setting, where we aim to bound the augmentation mutual information $I_\ell^{Z}(G;W)$. Let $z \in \mathcal{Z}$ be a fixed data point, and consider the random variables $G \sim \mathcal{D}_\mathcal{G}$ and $W \sim P_{W \mid z, G}$. Define the joint distribution $P_{G,W \mid z} \eqdef \mathcal{D}_\mathcal{G} \otimes P_{W \mid z, G}$ and let $P_{W \mid z} \eqdef \mathbb{E}_{G \sim \mathcal{D}_\mathcal{G}}[P_{W \mid z, G}]$ be the marginal. In the discrete case, applying the reverse Pinsker inequality (\ref{eq:reverse_pinsker}) to the joint distribution $P_{G, W \mid z}$ and the product distribution $\mathcal{D}_\mathcal{G} \otimes P_{W \mid z}$, we obtain
\begin{equation}
\label{eq:our_setting_reverse_pinsker}
I_\ell^{z}(G;W) = D_{KL}(P_{G, W \mid z}\| \mathcal{D}_\mathcal{G} \otimes P_{W \mid z}) \leq \frac{1 }{\delta_z}  \mathrm{TV}(P_{G, W \mid z}, \mathcal{D}_\mathcal{G} \otimes P_{W \mid z})^2.
\end{equation}
where $\delta_z \eqdef {\min}_{(g,w)} \mathcal{D}_\mathcal{G} (g) P_{W \mid z} (w) >0$ denotes the minimal joint probability under the product distribution $\mathcal{D}_\mathcal{G} \otimes P_{W \mid z}$. This upper bound is guaranteed under the assumption that both $\mathcal{D}_\mathcal{G}$ and $P_{W \mid z}$ have full support. The above inequality shows that if the joint distribution $P_{G, W \mid z}$ is close to the product $\mathcal{D}_\mathcal{G} \otimes P_{W \mid z}$ in total variation, then the augmentation mutual information is necessarily small. Intuitively, this means that if the hypothesis $W$ is nearly independent of the augmentation $G$, given the input $z$, then little information is leaked through augmentation, and $I_\ell^{z}(G;W)$ is small. This occurs naturally when the learning algorithm is approximately $\mathcal{G}$-invariant, or when the mapping $g \mapsto P_{W|z,g}$ is $C$-Lipschitz in total variation, i.e.,
$$\mathrm{TV}(P_{W|z,g}, P_{W|z,g'}) \leq  C d_\mathcal{Z}(gz,g'z), \forall g,g'\in \mathcal{G},$$

where $d_\mathcal{Z}$ is a metric defined on the space $\mathcal{Z}$ and $C>0$. 

While this Lipschitz condition in total variation is a strong regularity assumption, it is closely connected to well-established notions of algorithmic stability and approximate invariance. Specifically, it formalizes the requirement that small perturbations of the augmented input $gz$ induce only small changes in the distribution of learned hypotheses, a property that is implicitly enforced by stable training procedures, regularization, injected noise, or explicit invariance constraints. From this perspective, the assumption should be viewed as a sufficient, rather than necessary, condition that isolates the mechanism by which approximate invariance leads to small augmentation mutual information. Our analysis does not require this condition to hold universally; rather, it characterizes a regime in which the total variation distance between $P_{G, W \mid z}$ and $\mathcal{D}_\mathcal{G} \otimes P_{W \mid z}$ can be effectively controlled, thereby yielding meaningful bounds on $I_\ell^{z}(G;W)$.

Under such smoothness assumption, the variability of the conditional distribution  $P_{W|z,g}$ across different augmentations is controlled by how far the corresponding augmented inputs $gz$ and $g'z$ are in the data space. If the group transformation $\mathcal{G}$ has small diameter in the metric $d_\mathcal{Z}$, then all group elements induce only mild changes in the input, and hence the total variation between conditional distributions remains small. 

Formally, we have the following proposition

\begin{proposition}
\label{prop:bound_third_term}
Let $\mathcal{G}$ be a finite group acting on a finite input space $\mathcal{Z}$, let $\mathcal{W}$ be a finite parameter space, and  let $P_{W|z,g}$ denote the conditional distribution of $W$ given an input $z \in \mathcal{Z}$ and a group transformation $g \in \mathcal{G}$ . Suppose that for all $z \in \mathcal{Z}$, the mapping $g \mapsto P_{W|z,g}$ is $C$-Lipschitz in total variation, i.e.,
$$\mathrm{TV}(P_{W|z,g}, P_{W|z,g'}) \leq  C d_\mathcal{Z}(gz,g'z), \forall g,g'\in \mathcal{G}.$$
Then, each per-sample augmentation mutual information term appearing in Theorem~\ref{th:bu_like_bound_mi} is uniformly bounded as
$$I_\ell^{z_i}(G_j;W) \leq \frac{C^2\Delta_\mathcal{G}^2}{\delta_z} , \forall i\in [m] \text{ and } j\in [n].$$
where $\delta_z \eqdef {\min}_{(g,w)} \mathcal{D}_\mathcal{G} (g) P_{W \mid z} (w) >0$ denotes the minimal joint probability under $\mathcal{D}_\mathcal{G} \otimes P_{W \mid z}$. 
\end{proposition}
The proof of Proposition \ref{prop:bound_third_term} is provided in Appendix \ref{proof:bound_third_term}. 

This proposition provides a quantitative link between model sensitivity to augmentation and the geometry of the transformation group used for data augmentation. It demonstrates that the augmentation mutual information admits a natural control by two key factors: the smoothness of the learning algorithm encoded by $C$ and the group diameter $\Delta_\mathcal{G}$.  When the augmentation group $\mathcal{G}$ has small diameter, meaning its elements induce only mild perturbations of the input, the augmentation mutual information is inherently bounded. Moreover, if the learning algorithm is approximately $\mathcal{G}$-invariant, such that the conditional distribution  $P_{W|z,g}$ depends smoothly on $g$, the augmentation mutual information vanishes in the limit as the constant $C \to 0$. 

Importantly, this result should be interpreted as identifying a sufficient regime in which the augmentation mutual information term appearing in Theorem~\ref{th:bu_like_bound_mi} is uniformly controlled, rather than as a statement about the overall generalization behavior of the full learning algorithm. The bound in the proposition concerns the conditional distribution $P_{W|z,g}$, which describes the sensitivity of the learning algorithm to a single augmented example $gz$, and should be viewed as a local or per-sample property of the algorithm. By contrast, the generalization bound in Theorem~\ref{th:bu_like_bound_mi} is expressed in terms of the full posterior distribution $P_{W|E,S}$, which depends jointly on the entire training set $S$ and all augmentation randomness $E$. The proposition does not aim to control this full posterior directly. Instead, it isolates the contribution of augmentation randomness at the level of individual samples and transformations, and provides conditions under which this contribution remains small. In this sense, Proposition~\ref{prop:bound_third_term} should be understood as controlling one specific component of the overall bound, rather than characterizing the global stability or generalization behavior of the learning algorithm.

From this perspective, the proposition clarifies when augmentation randomness has limited impact: if the algorithm is stable or approximately invariant, then the joint distribution $P_{G, W \mid z}$ remains close to the product $\mathcal{D}_\mathcal{G} \otimes P_{W \mid z}$, yielding small augmentation mutual information. This insight supports the design of algorithms that are robust to augmentation randomness, either by enforcing smooth dependence on the group elements (e.g., via regularization or invariant architectures), or by limiting the augmentation group to low-diameter transformations. In particular, smooth augmentations such as small-angle rotations or pixel-level translations yield tighter bounds, while augmentations from groups with large diameter (e.g., wide-range affine transformations) may increase the augmentation mutual information unless explicitly mitigated by invariance-promoting mechanisms. 

This interpretation is consistent with existing results on exact invariance. For example, \cite{elesedy2022group} analyze the special case of exact invariance, where $P_{W|z,g}$ is completely independent of $g$ (i.e., $C=0$). Our result can be viewed as a quantitative relaxation of this setting: it characterizes how approximate invariance, measured through the Lipschitz constant $C$, interacts with the geometry of the transformation group to control the augmentation mutual information and its contribution to the generalization bound.

Finally, when training with SGD using data augmentation, the learned model typically becomes increasingly invariant to the applied transformations over the course of training. This suggests that the sensitivity constant $C$, which quantifies the dependence of $P_{W|z,g}$ on the specific augmentation $g$, may decrease as training progresses. Consequently, as the algorithm approaches approximate $\mathcal{G}$-invariance, the third term in the bound becomes small, leading to tighter generalization guarantees. From this perspective, the bound provides a theoretical lens through which data augmentation promotes generalization via training-induced invariance.

\subsubsection{Infinite and continuous case analysis}

In the finite case, the reverse Pinsker inequality allows one to control the KL-divergence in terms of the total variation distance, with constant depending on the minimum probability mass of the reference distribution. However, this approach fails in general for continuous probability distributions: even arbitrarily small total variation does not exclude the KL-divergence from being arbitrarily large or infinite. To recover a meaningful relationship in continuous spaces, one must impose additional regularity conditions that restrict how concentrated the joint distribution can be relative to a product reference measure.

A natural and mild regularity assumption, often satisfied in practical machine learning scenarios with randomized augmentations, is to assume that the joint conditional distribution $P_{G, W \mid z}$ is absolutely continuous with respect to the product measure $\mathcal{D}_\mathcal{G} \otimes P_{W \mid z, G}$, i.e., $$P_{G, W \mid z} \ll \mathcal{D}_\mathcal{G} \otimes P_{W \mid z},$$ 
and that the Radon–Nikodym derivative satisfies 
\begin{equation}
\label{eq:density_ratio_condition}
\frac{d P_{G,W \mid z}}{d(\mathcal{D}_\mathcal{G} \otimes P_{W \mid z})} \leq \beta \quad \text{for $\mathcal{D}_\mathcal{G} \otimes P_{W \mid z}$-almost everywhere}.    
\end{equation}

This bounded density ratio condition generalizes the role of the minimum probability mass in the discrete case and enables continuous spaces analogues of reverse Pinsker bounds to relate KL-divergence and total variation \citep{verdu2014total, sason2015reverse}. In particular, we have
\begin{equation}
I_\ell^{z}(G;W) = D_{KL}(P_{G, W \mid z}\| \mathcal{D}_\mathcal{G} \otimes P_{W \mid z}) \leq \frac{\sqrt{\beta} }{2}  \mathrm{TV}(P_{G, W \mid z}, \mathcal{D}_\mathcal{G} \otimes P_{W \mid z}).
\end{equation}
Under this assumption, we can derive a direct and explicit upper bound on the augmentation mutual information term $I_\ell^{z}(G;W)$, which plays a central role in the third term of Theorem~\ref{th:bu_like_bound_mi} for continuous settings. 

We have the following proposition

\begin{proposition}
\label{prop:bound_third_term_count_setting}
Let $\mathcal{G}$ be a compact group acting on the input space $\mathcal{Z}$, and  let $P_{W|z,g}$ denote the conditional distribution of $W$ given an input $z \in \mathcal{Z}$ and a group transformation $g \in \mathcal{G}$. Suppose  that for all $z \in \mathcal{Z}$, the mapping $g \mapsto P_{W|z,g}$ is $C$-Lipschitz in total variation, i.e.,
$$\mathrm{TV}(P_{W|z,g}, P_{W|z,g'}) \leq  C d_\mathcal{Z}(gz,g'z), \forall g,g'\in \mathcal{G}.$$
Assume that the joint conditional distribution $P_{G, W \mid z}$ is absolutely continuous with respect to the product measure $\mathcal{D}_\mathcal{G} \otimes P_{W \mid z, G}$, i.e., $P_{G, W \mid z} \ll \mathcal{D}_\mathcal{G} \otimes P_{W \mid z}$, and that the Radon–Nikodym derivative satisfies 
$$ \frac{d P_{G,W \mid z}}{d(\mathcal{D}_\mathcal{G} \otimes P_{W \mid z})} \leq \beta \quad \text{for $\mathcal{D}_\mathcal{G} \otimes P_{W \mid z}$-almost everywhere }.$$
Then, the augmentation mutual information term in Theorem~\ref{th:bu_like_bound_mi} satisfies
$$I_\ell^{z_i}(G_j;W) \leq \frac{\sqrt{\beta} }{2} C\Delta_\mathcal{G}, \forall i\in [m] \text{ and } j\in [n].$$
\end{proposition}

The proof of Proposition \ref{prop:bound_third_term_count_setting} proceeds analogously to that of Proposition \ref{prop:bound_third_term}. 

Proposition \ref{prop:bound_third_term_count_setting} extends the augmentation-sensitivity result of Proposition \ref{prop:bound_third_term} from finite to compact and potentially continuous groups. Unlike the discrete finite case, which depends quadratically on the group diameter and requires a minimal joint probability assumption (which can be arbitrarily small in high-dimensional settings), this continuous formulation relies on absolute continuity and yields a linear dependence on the group diameter $\Delta_\mathcal{G}$. The factor $\sqrt{\beta}$ reflects the degree of dependence between augmentations and learned parameters, while the constant $C$ encodes the smoothness of the learning algorithm. 

Under the bounded density ratio assumption \ref{eq:density_ratio_condition}, a direct consequence is the trivial bound $I_\ell^{z_i}(G_j;W) \leq \log \beta$, which corresponds to a worst-case control allowing arbitrarily irregular dependence on the group action. In contrast, our bound leverages the additional assumption that $g \mapsto P_{W|z,g}$ varies smoothly in total variation. This regularity allows one to propagate local sensitivity along group orbits and replace the purely information-theoretic $\log \beta$ dependence by a geometric quantity $C \Delta_\mathcal{G}$, yielding $I_\ell^{z_i}(G_j;W) \leq \sqrt{\beta} C\Delta_\mathcal{G}, /2$. Thus, while the $\log \beta$ bound remains tight in the absence of structure, the Lipschitz–geometric bound provides a strictly sharper, data-dependent control in regimes where training induces approximate invariance.
\section{Experiments}
\label{sec:experiments}
To illustrate the practical relevance of our theoretical results, we perform controlled experiments on the MNIST \citep{lecun1998mnist} and FashionMNIST \citep{xiao2017fashionmnist} datasets using convolutional neural networks (CNNs) trained with varying strengths of affine data augmentations (i.e., rotation angle or translation amplitude). The objective is to assess how accurately the proposed information-theoretic generalization bound in Theorem \ref{th:bu_like_bound_mi} captures the behavior of models as the augmentation strength varies. Specifically, we estimate key quantities appearing in our bounds, including the KL-divergence between the original data distribution and its augmentation induced counterpart, as well as several mutual information terms. The KL-divergence is estimated via a density ratio estimation approach based on a discriminator classifier \citep{nguyen2007estimating, choi2021featurized}, while the mutual information terms are computed using Mutual Information Neural Estimation (MINE) \citep{belghazi2018mine}, providing a flexible and scalable data driven estimation framework. All experiments are averaged across random seeds and use caching for efficiency and reproducibility. Details of the experimental setup and estimation procedures are provided in Appendix~\ref{detail:experiments}.

Although both the bounds in Theorem \ref{th:xu_like_bound_mi} and Theorem \ref{th:bu_like_bound_mi} provide valid information-theoretic guarantees, we focus our empirical analysis on the last one without loss of generality. This restriction is motivated by the fact that both bounds display qualitatively similar dependence on the augmentation strength, tracking the generalization gap in a comparable manner and differing mainly by scaling constants in their information terms. This choice avoids redundancy while maintaining full representational value for both formulations.

Figure~\ref{fig:bu_bound_estimation} summarizes the empirical results, illustrating the evolution of the estimated generalization bounds and their constituent information-theoretic terms as the augmentation strength increases for both the MNIST and FashionMNIST datasets. Each panel shows the empirical generalization gap, the estimated bound, and its three constituent terms as functions of augmentation strength. Overall, the results confirm that the proposed bound in Theorem \ref{th:bu_like_bound_mi} closely tracks the empirical generalization gap, exhibiting a consistent and interpretable dependence on the level of augmentation. In particular, moderate augmentations reduce the mutual information between model parameters and training samples, yielding tighter bounds and improved generalization performance. However, as the augmentation strength becomes excessive, the KL-divergence between the original and augmented distributions grows, indicating a distributional mismatch that diminishes the regularization benefits of augmentation. Similar qualitative behavior is observed across both datasets. These findings highlight the capacity of our bound to quantitatively capture the trade-off between invariance and distributional shift induced by data augmentation.

\begin{figure}[htbp]
  \centering
  \begin{subfigure}[b]{0.40\linewidth}
    \centering
    \includegraphics[width=\linewidth]{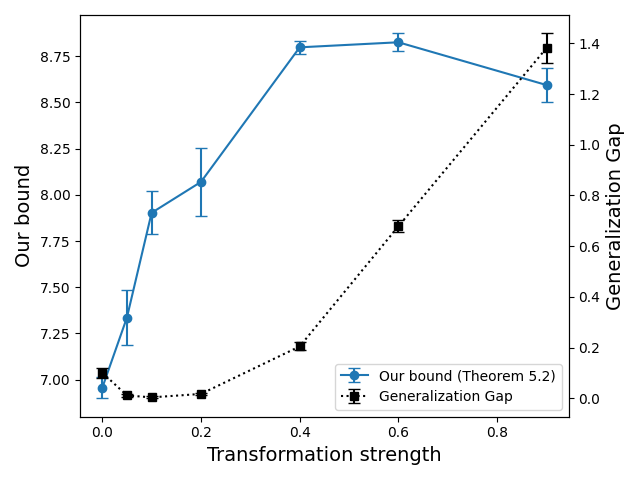}
    \caption{MNIST bound}
    \label{fig:generalization_bounds_MNIST}
  \end{subfigure}
  \quad
    \begin{subfigure}[b]{0.40\linewidth}
    \centering
    \includegraphics[width=\linewidth]{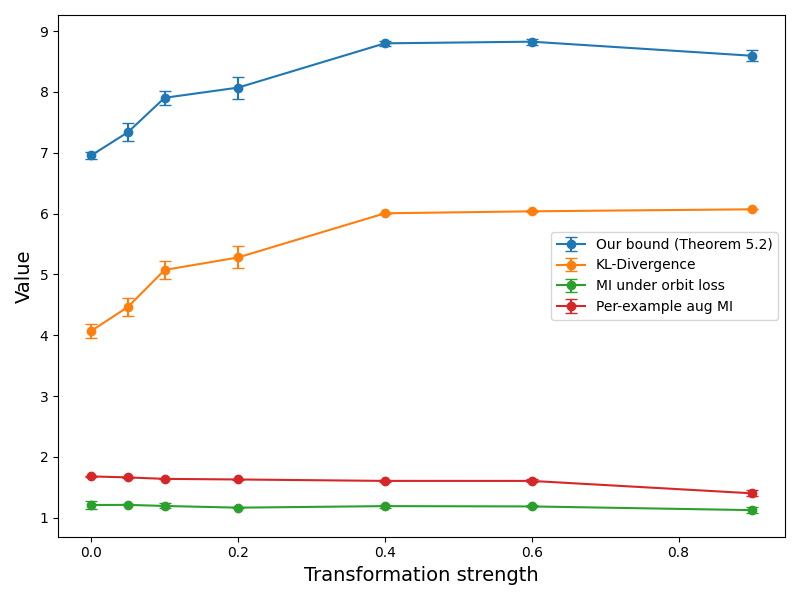}
    \caption{Terms contribution MNIST bound}
    \label{fig:bu_terms_contibution_MNIST}
  \end{subfigure}

  \vspace{0.6em}

  \begin{subfigure}[b]{0.40\linewidth}
    \centering
    \includegraphics[width=\linewidth]{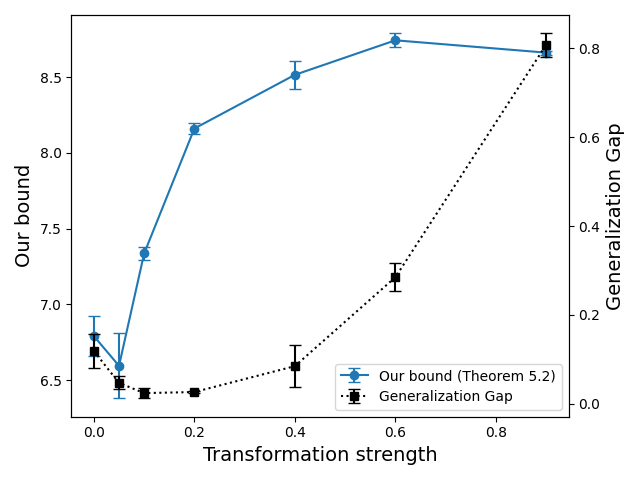}
    \caption{FashionMNIST bound}
    \label{fig:generalization_bounds_Fashion}
  \end{subfigure}
  \quad
    \begin{subfigure}[b]{0.40\linewidth}
    \centering
    \includegraphics[width=\linewidth]{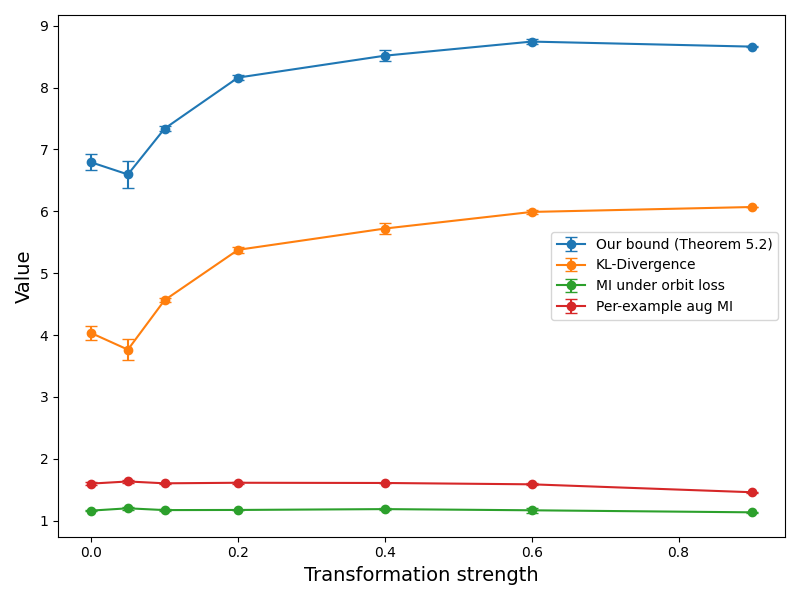}
    \caption{Terms contribution FashionMNIST bound}
    \label{fig:bu_terms_contibution_Fashion}
  \end{subfigure}

  \caption{Empirical evaluation of the generalization bound in Theorem \ref{th:bu_like_bound_mi} on MNIST and FashionMNIST with affine augmentations. Specifically, the transformation transforms.RandomAffine(degrees = strength × 10, translate = (strength, strength)) randomly rotates each image within ±(10 × strength) degrees and translates it by up to a fraction strength of the image width and height. (a, c) Estimated bound versus empirical generalization gap as a function of augmentation level. The bound increases consistently with transformation strength and closely follows the observed generalization behavior. (b, d) Contribution of individual terms; KL-divergence, orbit-averaged mutual information, and per-example augmentation mutual information; showing that the KL-divergence term dominates while mutual information components remain relatively stable across augmentation levels.}
  \label{fig:bu_bound_estimation}
\end{figure}

\section{Conclusion}
\label{sec:conclusion}
In this paper, we developed an information-theoretic framework to analyze the generalization properties of models trained with data augmentation. Building on mutual information–based bounds \citep{xu2017information, bu2020tightening}, we introduced the orbit-averaged loss and derived a decomposition of the generalization gap into distributional divergence, dependence on training data, and augmentation sensitivity. To capture augmentation geometry, we defined the group diameter, which  bounds all three terms and makes explicit the trade-off between increased invariance and the combined costs of distributional bias and augmentation-induced sensitivity. As future work, we aim to extend our framework to two promising directions: (1) learnable augmentation distributions, where the transformation law is optimized rather than fixed \citep{hounie2023automatic, liu2021direct}, and (2) generative data augmentation (GDA), where deep generative models produce synthetic training examples. Incorporating the inherent randomness of generative models into information-theoretic bounds would unify these settings and advance the theoretical understanding of augmentation beyond invariance learning \citep{azizi2023synthetic, zheng2023toward}.

\section*{Acknowledgments}
This work is supported by the Collaborative Research and Development Grant from Beneva and NSERC, CRDPJ 529584.

\bibliographystyle{unsrt}  
\bibliography{references}  

\appendix                       

\section{Technical Lemmas}     

\begin{lemma}{\cite[Corollary 4.14]{boucheron2003concentration}}
\label{lem:Donsker-Varadhan-1}
Let $P$  and $Q$ be two probability distributions on the same space $\mathcal{Z}$ such that $Q$ is absolutely continuous with respect to $P$. Let $f$ be a real valued function and $\lambda \in \mathbb{R}$ such that $\mathbb{E}_P[e^{\lambda(f(Z) - \mathbb{E}_P(f(Z)))}] < \infty$. Then, we have
$$ \mathbb{E}_Q[f(Z)] - \mathbb{E}_P[f(Z)] \leq \frac{1}{\lambda}(\log  \mathbb{E}_P[ e^{\lambda(f(Z) - \mathbb{E}_P[f(Z)] )}] + D_{\text{KL}}(Q|| P)).$$
\end{lemma}
\begin{proof}
See \cite{boucheron2003concentration}.
\end{proof}

\begin{lemma}{\cite[Corollary 4.15]{boucheron2003concentration}}
\label{lem:Donsker-Varadhan-2}
Let $P$  and $Q$ be two probability distributions on the same space $\mathcal{Z}$ such that $P$ is absolutely continuous with respect to $Q$. Then, we have
$$ D_{\text{KL}}(Q|| P)) = \underset{f \in \mathcal{F}}{\sup } \{ \mathbb{E}_P[f(Z)] - \log\mathbb{E}_Q [e^{f(Z)}] \};$$
where the supremum is over all measurable functions $\mathcal{F} = \{f : \mathcal{Z} \to \mathbb{R} \mid \mathbb{E}_Q [e^{f(Z)}]  < \infty \}$.
\end{lemma}
\begin{proof}
See \cite{boucheron2003concentration}.
\end{proof}

\begin{lemma}[ Decoupling Estimate Lemma, \citep{xu2017information}]
\label{lem:decoupling-estimate-lemma}
Let $X$ and $Y$ be two random variables with joint distribution
$P_{X,Y}$. Let $\bar{X}$ be an independent copy of $X$, and $\bar{Y}$ an independent copy of $Y$  such that $P_{\bar{X},\bar{Y}} =  P_{X} \otimes P_{Y}$. Let $f : X \times Y \to \mathbb{R}$ be an arbitrary real-valued function. If $f(\bar{X}, \bar{Y})$ is $\sigma-$sub-Gaussian under $P_{\bar{X},\bar{Y}}$, then 
$$|\mathbb{E}[f(X,Y)] - \mathbb{E}[f(\bar{X},\bar{Y})]| \leq \sqrt{2\sigma^2I(X,Y)}.$$
\end{lemma}
\begin{proof}
From Lemma \ref{lem:Donsker-Varadhan-2} and the subgaussian assupmtion on $f(\bar{X}, \bar{Y})$, we have for any $\lambda \in \mathbb{R}$,
\begin{align}
    \nonumber D_{KL}(P_{X,Y} || P_{X} \otimes P_{Y}) &\geq \mathbb{E}[\lambda f(X, Y)] - \log \mathbb{E}[e^{\lambda f(\bar{X},\bar{Y})}] \\
    &\geq \lambda(\mathbb{E}[f(X,Y)] - \mathbb{E}[f(\bar{X},\bar{Y})]) - \frac{\lambda^2 \sigma^2}{2} \label{ineq:ineq2},
\end{align}
Inequality \ref{ineq:ineq2} gives a nonnegative probabla in $\lambda$, whose discriminant must be nonpositive, then
\begin{align}
    \nonumber |\mathbb{E}[f(X,Y)] - \mathbb{E}[f(\bar{X},\bar{Y})]| &\leq \sqrt{2\sigma^2D_{KL}(P_{X,Y} || P_{X} \otimes P_{Y})} \\
    \nonumber &= \sqrt{2\sigma^2I(X,Y)}.
\end{align}
\end{proof}

\begin{lemma}
\label{lem:mi_kl_equivalence}
Let $X$ and $Y$ be two random variables with joint distribution
$P_{X,Y}$. Then the mutual information between $X$ and $Y$ satisfies
$$I(X;Y) \;=\; \mathbb{E}_{X} \left[ D_{\mathrm{KL}}( P_{Y|X} || P_Y ) \right].$$
\end{lemma}

\begin{proof}
Assume $P_{X}$, $P_{Y}$, and $P_{X,Y}$ admit densities $p_{X}$, $p_{Y}$, and $p_{X,Y}$ respectively. By definition, the mutual information is
$$I(X;Y) = D_{\mathrm{KL}}( P_{X,Y} || P_X \otimes P_Y ) =  \iint p(x,y) \, \log \frac{p(x,y)}{p(x)p(y)} \, dx \, dy.$$
Using the factorization $p(z,w) = p(z) p(w|z)$, we have:
$$ \frac{p(x,y)}{p(x)p(y)} = \frac{p(x) p(y|x)}{p(z) p(y)} = \frac{p(x|y)}{p(y)}. $$

Substituting, we get
\begin{align}
\nonumber I(X;Y) &= \iint p(x) p(y|x) \, \log \frac{p(y|x)}{p(y)} \, dy \, dx \\
&= \int p(x) \left[ \int p(y|x) \, \log \frac{p(y|x)}{p(y)} \, dy \right] dx \label{eq:Fubini}\\
\nonumber &= \mathbb{E}_{X} \left[ D_{\mathrm{KL}}( P_{Y|X} || P_Y ) \right],
\end{align}
where the equality \ref{eq:Fubini} is obtained from Fubini's theorem.

This concludes the proof.
\end{proof}

\begin{lemma}[Data processing inequality]
Given random variables $X$, $Y$, $Z$ and the Markov chain 
$$X \to Y \to Z,$$ 
where $X \indep Z | Y$, then we have
$$I(X, Z) \leq \min\{ I(X, Y); I(Y, Z)\}.$$
\end{lemma}
\begin{proof}
By the Chain rule of mutual information, we know that $I(X;Y,Z)$ can be decomposed in two ways 
$$I(X;Y,Z) = I(X;Z)+I(X;Y |Z) = I(X; Y ) + I(X; Z | Y ) $$
Because $I(X;Z|Y) = 0$ by assumption $X \indep Z | Y$, we have that $I(X;Z) + I(X;Y |Z) = I(X;Y)$. Since mutual information is non-negative, we get that $I(X; Z) \leq I(X; Y )$.
Also,
$$I(Z;Y,X) = I(Z;X)+I(Z;Y |X) = I(Z; Y ) + I(Z; X | Y ),$$
with $I(Y;Z|X) = 0$, we have $I(X, Z) \leq I(Y, Z)$
\end{proof}

\clearpage
\section{Proofs}     

\subsection{Proof of Theorem \ref{th:xu_like_bound_mi} }
\label{proof:xu_like_bound_mi}
We recall the theorem
\begin{theorem*} 
Suppose assumptions \ref{assumption_1} and \ref{assumption_2} are satisfied, then 
$$ 
\lvert \mathrm{gen}(\mathcal{D}, P_{W|S,E}) \rvert \leq R \left( \sqrt{2D_{\mathrm{KL}} (\mathcal{D}\Vert \mathcal{D}_\mathcal{G} \circ \mathcal{D})} + \sqrt{\frac{2}{m}I_{\ell_{\mathcal{G}}}(S;W)} + \frac{1}{m} \underset{S \sim \mathcal{D}^m}{\mathbb{E}} \left[ \sum_{i=1}^{m}  \sqrt{\frac{2}{n}I_{\ell}^{Z_i}(E;W)}\right] \right);
$$
\end{theorem*}
\begin{proof}
To prove the theorem, we begin by decomposing the expected generalization gap into three interpretable components. Specifically, we write
\begin{align}
\ \nonumber \mathrm{gen}(\mathcal{D}, P_{W|S,E}) = {}& \underset{S \sim \mathcal{D}^m}{\mathbb{E}} \underset{E \sim \mathcal{D}_\mathcal{G}^{n}}{\mathbb{E}} \underset{W\sim P_{W|S,E}}{\mathbb{E}} [L_{\mathcal{D}}(W) - L_{E \circ S }(W)] \\
     = {}& \underset{S \sim \mathcal{D}^m}{\mathbb{E}} \underset{E \sim \mathcal{D}_\mathcal{G}^{n}}{\mathbb{E}} \underset{W\sim P_{W|S,E}}{\mathbb{E}} [L_{\mathcal{D}}(W) - L_{\mathcal{D}_\mathcal{G} \circ \mathcal{D}}(W)] \label{proof:true_risks_gap} \\
      {}& +   \underset{S \sim \mathcal{D}^m}{\mathbb{E}} \underset{E \sim \mathcal{D}_\mathcal{G}^{n}}{\mathbb{E}} \underset{W\sim P_{W|S,E}}{\mathbb{E}} [L_{\mathcal{D}_\mathcal{G} \circ \mathcal{D}}(W) - L_{\mathcal{D}_\mathcal{G} \circ S}(W)] \label{proof:new_loss_gap}\\
      {}& +  \underset{S \sim \mathcal{D}^m}{\mathbb{E}} \underset{E \sim \mathcal{D}_\mathcal{G}^{n}}{\mathbb{E}} \underset{W\sim P_{W|S,E}}{\mathbb{E}} [L_{\mathcal{D}_\mathcal{G} \circ S }(W) - L_{E \circ S }(W)] \label{proof:new_appox_cost},
\end{align} 
where all expectations are taken over the joint sampling of training data $S \sim \mathcal{D}^m$, transformation samples $E \sim \mathcal{D}_\mathcal{G}^n$, and model output $W \sim P_{W|S,E}$.

To obtain a bound on the expected generalization gap $\mathrm{gen}(\mathcal{D}, P_{W|S,E})$, we bound each of the three terms in this decomposition separately.

The first term in the decomposition, Equation (\ref{proof:true_risks_gap}), captures the distribution shift introduced by data augmentation. It measures the discrepancy between the true risk evaluated under the original data distribution $\mathcal{D}$ and the risk evaluated under the augmented distribution $\mathcal{D}_\mathcal{G}$. This term can be controlled using the Kullback–Leibler (KL) divergence $D_{\mathrm{KL}}(\mathcal{D}\Vert \mathcal{D}_\mathcal{G} \circ \mathcal{D})$ between the original and augmented distributions. Let $\lambda \in \mathbb{R} $ and $w \in \mathcal{W}$. Donsker-Varadhan Lemma  (\ref{lem:Donsker-Varadhan-1}) implies that
\begin{align}
 \nonumber   L_{\mathcal{D}}(w) - L_{\mathcal{D}_\mathcal{G} \circ \mathcal{D}}(w) &\leq \frac{1}{\lambda}(\log \underset{Z \sim \mathcal{D}_\mathcal{G} \circ \mathcal{D} } {\mathbb{E}} e^{\lambda(\ell(w,Z) - \mathbb{E} \ell(w,Z) )} + D_{\mathrm{KL}}(\mathcal{D}\Vert \mathcal{D}_\mathcal{G} \circ \mathcal{D}))\\
    &\leq \frac{1}{\lambda}(\frac{\lambda^2 R^2}{2} + D_{\mathrm{KL}}(\mathcal{D}\Vert \mathcal{D}_\mathcal{G} \circ \mathcal{D})) \label{eq:first_term}.
\end{align}
The last inequality is obtained since $\ell$ is $R$-sub-Gaussian under $\mathcal{D}_\mathcal{G} \circ \mathcal{D}$ for all $w \in \mathcal{W}$, i.e., assumption \ref{assumption_1} . The right side of the last equation reaches its minimum when $\lambda = \frac{\sqrt{2}}{R} \sqrt{D_{\mathrm{KL}}(\mathcal{D}, \mathcal{D}_\mathcal{G} \circ \mathcal{D})}$, then we have
 $$L_{\mathcal{D}}(w) - L_{\mathcal{D}_\mathcal{G} \circ \mathcal{D}}(w) \leq R \sqrt{2D_{\mathrm{KL}}(\mathcal{D}\Vert \mathcal{D}_\mathcal{G} \circ \mathcal{D})}.$$

Since the inequality \ref{eq:first_term} is also valid when $\lambda$ is negative, this implies that 
$$L_{\mathcal{D}_\mathcal{G} \circ \mathcal{D}}(w) - L_{\mathcal{D}}(w) \leq - \frac{1}{\lambda}(\frac{\lambda^2 R^2}{2} + D_{\mathrm{KL}}(\mathcal{D}\Vert \mathcal{D}_\mathcal{G} \circ \mathcal{D})) , \hspace{1cm} \forall \lambda < 0.$$

Hence, we also have
$$L_{\mathcal{D}_\mathcal{G} \circ \mathcal{D}}(w) - L_{\mathcal{D}}(w) \leq R \sqrt{2D_{\mathrm{KL}}(\mathcal{D}\Vert \mathcal{D}_\mathcal{G} \circ \mathcal{D})}.$$

We conclude that 
\begin{align*}
\left|\underset{S \sim \mathcal{D}^m}{\mathbb{E}} \underset{E \sim \mathcal{D}_\mathcal{G}^{n}}{\mathbb{E}} \underset{W\sim P_{W|S,E}}{\mathbb{E}} [L_{\mathcal{D}}(W) - L_{\mathcal{D}_\mathcal{G} \circ \mathcal{D}}(W)] \right| & \leq R \sqrt{2D_{\mathrm{KL}}(\mathcal{D}\Vert \mathcal{D}_\mathcal{G} \circ \mathcal{D})}.
\end{align*}

For the second term (\ref{eq:new_loss_gap}), we begin by establishing that the orbit-averaged loss function $\ell_\mathcal{G}$ inherits the sub-Gaussian property of the original loss function $\ell$.

Recall that the orbit-averaged loss is defined as
$$\ell_\mathcal{G}(w, Z) \eqdef \underset{G \sim \mathcal{D}_\mathcal{G} }{\mathbb{E}} \, \ell (w, GZ)$$

for $Z \sim \mathcal{D}$ and any $w \in \mathcal{W}$. We want to show that $\ell_\mathcal{G}(w, Z)$ is $R$-sub-Gaussian under $Z \sim \mathcal{D}$, given that $\ell(w, Z)$ is $R$-sub-Gaussian under $Z \sim \mathcal{D}_\mathcal{G} \circ \mathcal{D}$.

Let $\lambda \in \mathbb{R}$. By Jensen's inequality and the convexity of the exponential function, we have
\begin{align*}
    \underset{Z \sim \mathcal{D}} {\mathbb{E}} e^{\lambda(\ell_\mathcal{G} (w,
    Z) - \underset{Z'\sim \mathcal{D}} {\mathbb{E}} \ell_\mathcal{G} (w, Z') )} &=  \underset{Z \sim \mathcal{D}} {\mathbb{E}} e^{\lambda(\underset{G \sim \mathcal{D}_\mathcal{G} }{\mathbb{E}} \ell (w, GZ) - \underset{Z' \sim \mathcal{D}} {\mathbb{E}} \underset{G' \sim \mathcal{D}_\mathcal{G} }{\mathbb{E}} \ell (w, G'Z') )} \\
    &\leq  \underset{Z \sim \mathcal{D}} {\mathbb{E}} \underset{G \sim \mathcal{D}_\mathcal{G} }{\mathbb{E}} e^{\lambda( \ell (w, GZ) - \underset{Z' \sim \mathcal{D}} {\mathbb{E}}  \underset{G' \sim \mathcal{D}_\mathcal{G} }{\mathbb{E}}\ell (w, G'Z') )} \\
    &=  \underset{Z' \sim \mathcal{D}_\mathcal{G} \circ \mathcal{D}} {\mathbb{E}} e^{\lambda( \ell (w, Z') - \underset{Z' \sim \mathcal{D}_\mathcal{G} \circ \mathcal{D}} {\mathbb{E}}\ell (w, Z') )} \\
    &\leq e^{\frac{\lambda^2 R^2}{2}};
\end{align*} 
Thus, if $\ell(w, Z)$ is $R$-sub-Gaussian under the augmented distribution $\mathcal{D}_\mathcal{G} \circ \mathcal{D}$, then the orbit-averaged loss $\ell_\mathcal{G}(w, Z)$ is also $R$-sub-Gaussian under the original distribution $\mathcal{D}$. This result implies that averaging the loss over random transformations drawn from $\mathcal{D}_\mathcal{G}$ does not increase its variability in the sub-Gaussian sense. In other words, orbit-averaging preserves the concentration properties of the loss. Therefore, we may apply the bound in Theorem \ref{th:raginsky_bound} to conclude that 
\begin{align*}
\left| \underset{S \sim \mathcal{D}^m}{\mathbb{E}} \underset{E \sim \mathcal{D}_\mathcal{G}^{n}}{\mathbb{E}} \underset{W\sim P_{W|S,E}}{\mathbb{E}} [L_{\mathcal{D}_\mathcal{G} \circ \mathcal{D}}(W) - L_{\mathcal{D}_\mathcal{G} \circ S}(W)] \right| =& \left| \underset{S \sim \mathcal{D}^m}{\mathbb{E}} \underset{W\sim P_{W|S}}{\mathbb{E}} [L_{\mathcal{D}_\mathcal{G} \circ \mathcal{D}}(W) - L_{\mathcal{D}_\mathcal{G} \circ S}(W)] \right| \\
\leq & \sqrt{\frac{2R^2}{m}I_{\ell_{\mathcal{G}}}(S;W)}. 
\end{align*}

Now, let us bound the third term in the decomposition of the expected generalization gap (\ref{proof:new_appox_cost}). Fix $S=s$, then we have
\begin{align*}
     & \left| \underset{E \sim \mathcal{D}_\mathcal{G}^{n}}{\mathbb{E}} \underset{W\sim P_{W|S=s,E}}{\mathbb{E}}  [ L_{\mathcal{D}_\mathcal{G} \circ s }(W) - L_{E \circ s}(W)] \right|\\
    &= \left| \underset{E \sim \mathcal{D}_\mathcal{G}^{n}}{\mathbb{E}} \underset{W\sim P_{W|S=s,E}}{\mathbb{E}}[\frac{1}{m} \sum_{i=1}^{m}  \underset{G \sim \mathcal{D}_\mathcal{G} }{\mathbb{E}}[\ell (W, Gz_i)] - \frac{1}{mn} \sum_{i=1}^{m} \sum_{j=1}^{n} \ell (W, G_{j}z_{i})] \right| \\
    &\leq \frac{1}{m} \sum_{i=1}^{m} \left| \underset{E \sim \mathcal{D}_\mathcal{G}^{n}}{\mathbb{E}} \underset{W\sim P_{W|S=s,E}}{\mathbb{E}}[ \underset{G \sim \mathcal{D}_\mathcal{G} }{\mathbb{E}}[\ell (W, Gz_i)] - \frac{1}{n} \sum_{j=1}^{n} \ell (W, G_{j}z_{i})] \right|.
\end{align*}
For $w \in \mathcal{W}$ and for each $i \in [m]$, define the function 
$$ f(e, w, z_i) = \frac{1}{n} \sum_{j=1}^{n} \ell (w, g_{j}z_{i}), \hspace{0.5cm} \text{where} \hspace{0.2cm} e=\{g_1, ..., g_n \}.$$
By Assumption \ref{assumption_2}, for all $w \in \mathcal{W}$ and $z_i \in \mathcal{Z}$, the random variable $\ell(w,Gz_i)$ is $R$-sub-Gaussian under $ \mathcal{D}_\mathcal{G}$. Hence, by standard properties of sub-Gaussian averages, $f(E,w, z_i)$ is $R / \sqrt{n}$-sub-Gaussian with respect to the randomness of $E \sim \mathcal{D}_\mathcal{G}^n$.

Applying the decoupling estimate lemma (\ref{lem:decoupling-estimate-lemma}) to the function $f(E,w,z_i)$, we have
\begin{align*}
      \left| \underset{E \sim \mathcal{D}_\mathcal{G}^{n}}{\mathbb{E}} \underset{W\sim P_{W|S=s,E}}{\mathbb{E}}[L_{\mathcal{D}_\mathcal{G} \circ s }(W) - L_{E \circ s}(W)] \right|
    &\leq \frac{1}{m} \sum_{i=1}^{m} \sqrt{\frac{2R^2}{n}I^{z_i}_\ell(E;W)}.
\end{align*}
where $I^{z_i}(E;W)$ is the mutual information between the set of transformations $E$ and the output $W$ of the learning algorithm where $E$ is applied to a fixed sample $z_i$.

Taking the expectation \textit{w.r.t.}\ $S \sim \mathcal{D}^m$ gives the result.

Combine the three bounds and factor out \(R\) to obtain the inequality in the statement of the theorem.

\end{proof}

\subsection{Proof of Theorem \ref{th:bu_like_bound_mi} }
\label{proof:bu_like_bound_mi}
We recall the theorem
\begin{theorem*} 
Suppose assumptions \ref{assumption_1} and \ref{assumption_2} are satisfied, then  
$$ 
\lvert \mathrm{gen}(\mathcal{D}, P_{W|S,E}) \rvert \leq R \left( \sqrt{2D_{\mathrm{KL}} (\mathcal{D}\Vert  \mathcal{D}_\mathcal{G} \circ \mathcal{D})} + \frac{1}{m} \sum_{i=1}^m\sqrt{2I_{\ell_\mathcal{G}}(Z_i;W)} + \frac{1}{mn} \underset{S \sim \mathcal{D}^m}{\mathbb{E}} \left[ \sum_{i=1}^{m} \sum_{j=1}^{n} \sqrt{2I_\ell^{Z_i}(G_j;W)}\right]\right)\, .
$$
\end{theorem*}
\begin{proof}
As in the proof of Theorem~\ref{th:xu_like_bound_mi} we decompose the expected generalization gap into three terms (distribution shift, orbit-loss generalization, and augmentation-sampling approximation). The first term is handled exactly as before (Donsker–Varadhan Lemma + sub-Gaussian tail) and yields
\begin{align*}
\left|\underset{S \sim \mathcal{D}^m}{\mathbb{E}} \underset{E \sim \mathcal{D}_\mathcal{G}^{n}}{\mathbb{E}} \underset{W\sim P_{W|S,E}}{\mathbb{E}} [L_{\mathcal{D}}(W) - L_{\mathcal{D}_\mathcal{G} \circ \mathcal{D}}(W)] \right| & \leq R \sqrt{2D_{\mathrm{KL}}(\mathcal{D}\Vert \mathcal{D}_\mathcal{G} \circ \mathcal{D})}.
\end{align*}

For the second term, we apply the per-sample refinement in Theorem~\ref{bound:bu_bound} to the algorithm $P_{W\mid S}$ and the orbit-averaged loss $\ell_\mathcal G$ to have 
\begin{align*}
\left| \underset{S \sim \mathcal{D}^m}{\mathbb{E}} \underset{E \sim \mathcal{D}_\mathcal{G}^{n}}{\mathbb{E}} \underset{W\sim P_{W|S,E}}{\mathbb{E}} [L_{\mathcal{D}_\mathcal{G} \circ \mathcal{D}}(W) - L_{\mathcal{D}_\mathcal{G} \circ S}(W)] \right| \leq & \frac{1}{m} \sum_{i=1}^m\sqrt{2I_{\ell_\mathcal{G}}(Z_i;W)}. 
\end{align*}

We bound the remaining term. Conditioning on $S=s$, we have
\begin{align*}
    & \left| \underset{E \sim \mathcal{D}_\mathcal{G}^{n}}{\mathbb{E}} \underset{W\sim P_{W|S=s,E}}{\mathbb{E}} [ L_{\mathcal{D}_\mathcal{G} \circ s }(W) - L_{E \circ s}(W)] \right|\\
    &= \left| \underset{E \sim \mathcal{D}_\mathcal{G}^{n}}{\mathbb{E}} \underset{W\sim P_{W|S=s,E}}{\mathbb{E}}[\frac{1}{m} \sum_{i=1}^{m}  \underset{G \sim \mathcal{D}_\mathcal{G} }{\mathbb{E}}[\ell (W, Gz_i)] - \frac{1}{mn} \sum_{i=1}^{m} \sum_{j=1}^{n} \ell (W, G_{j}z_{i})] \right| \\
    &\leq \frac{1}{m} \sum_{i=1}^{m} \left| \underset{E \sim \mathcal{D}_\mathcal{G}^{n}}{\mathbb{E}} \underset{W\sim P_{W|S=s,E}}{\mathbb{E}}[ \underset{G \sim \mathcal{D}_\mathcal{G} }{\mathbb{E}}[\ell (W, Gz_i)] - \frac{1}{n} \sum_{j=1}^{n} \ell (W, G_{j}z_{i})] \right|\\
    &= \frac{1}{m} \sum_{i=1}^{m} \left| \frac{1}{n} \sum_{i=1}^{n} \left( \underset{\tilde{W}, \tilde{G}}{\mathbb{E}}\ell (W, Gz_i) - \underset{\tilde{W}, \tilde{G}}{\mathbb{E}}\ell (W, G_{j}z_{i}) \right) \right|\\
    &\leq \frac{1}{mn} \sum_{i=1}^{m} \sum_{j=1}^{n}  \left|  \underset{\tilde{W}, \tilde{G}}{\mathbb{E}}\ell (W, Gz_i) - \underset{W, G_j}{\mathbb{E}}\ell (W, G_{j}z_{i}) \right|.
\end{align*}
By applying the decoupling estimate lemma (\ref{lem:decoupling-estimate-lemma}) with $X = G$, $Y = W$ we have 
\begin{align*}
      \left| \underset{E \sim \mathcal{D}_\mathcal{G}^{n}}{\mathbb{E}} \underset{W\sim P_{W|S=s,E}}{\mathbb{E}}[L_{\mathcal{D}_\mathcal{G} \circ s }(W) - L_{E \circ s}(W)] \right|
    &\leq \frac{1}{mn} \sum_{i=1}^{m} \sum_{j=1}^{n} \sqrt{2R^2I^{z_i}(G_j;W)}.
\end{align*}
where $I^{z_i}(G_j;W)$ is the mutual information between the transformation $G_j$ and the output $W$ of the learning algorithm where $G_j$ is applied to a fixed sample $z_i$.

Averaging \textit{w.r.t.}\ $S \sim \mathcal{D}^m$  gives the result.

Combine the three bounds and factor out \(R\) to obtain the inequality in the statement of the theorem.

\end{proof}

\subsection{Proof of Corollary \ref{cor:gaussian_mn_bound} }
\label{proof:gaussian_mn_bound}
\begin{corollary*}
Consider the gaussian setting. Then,
\begin{align*}
 \lvert \mathrm{gen}(\mathcal{D}, P_{W|S,E}) \rvert &\leq \sqrt{d \left( \frac{s^2}{s^2+t^2} - 1 - \log \frac{s^2}{s^2+t^2}  \right) } + \sqrt{d \log \frac{m(s^2+t^2 /mn) + m^2\nu^2}{(m-1)s^2 + t^2/n + m^2\nu^2} } \\
 &+ \sqrt{d \log \left(  1 + \frac{t^2/n^2}{(n-1)/n^2t^2 + (m-1)/m^2s^2 + \nu^2} \right)} .   
\end{align*}
\end{corollary*}
\begin{proof}
 We evaluate each term appearing in Theorem \ref{th:bu_like_bound_mi} for the specified Gaussian setting and substitute into the theorem. 

 The distribution of $Z$ is $\mathcal{D} = \mathcal{N}(\mu, s^2I_d)$. After augmentation (additive noise $G \sim \mathcal{N}(\mu, t^2I_d)$), the augmentation distribution is
 $$\mathcal{D}_\mathcal{G} \circ \mathcal{D}  = \mathcal{N}(\mu, (s^2 + t^2)I_d).$$
The KL divergence between two Gaussians with identical mean and covariances $\Sigma_1\eqdef s^2I_d$, $\Sigma_2\eqdef (s^2 + t^2)I_d$ is 
$$D_{\mathrm{KL}} (\mathcal{N}(\mu, \Sigma_1), \mathcal{N}(\mu, \Sigma_2)) = \frac{1}{2}(\text{tr}(\Sigma_2^{-1} \Sigma_1) - d - \log \det (\Sigma_2^{-1} \Sigma_1))$$
Because $\Sigma_2^{-1} \Sigma_1 = \frac{s^2}{s^2 + t^2}I_d$, this simplifies to
$$D_{\mathrm{KL}} (\mathcal{D}\Vert  \mathcal{D}_\mathcal{G} \circ \mathcal{D}) = \frac{d}{2}\left( \frac{s^2}{s^2 + t^2} -1 - \log \frac{s^2}{s^2 + t^2} \right).$$
Hence, the first term of the theorem is 
$$ \sqrt{2D_{\mathrm{KL}} (\mathcal{D}\Vert  \mathcal{D}_\mathcal{G} \circ \mathcal{D})} = \sqrt{d\left( \frac{s^2}{s^2 + t^2} -1 - \log \frac{s^2}{s^2 + t^2} \right)}.$$
Now, we compute the mutual information  $I_{\ell_\mathcal{G}}(Z_1, W)$ between the single sample $Z_1$ and the learned parameter $W$. Because covariance matrices are isotropic (scalar multiples of $I_d$), we can compute the scalar mutual information and multiply by $d$.

We have
$$ W = \frac{1}{mn} \sum_{i=1}^{m} \sum_{j=1}^{n}  (Z_i + G_j) + \epsilon = \frac{1}{m} \sum_{i=1}^{m} Z_i + \frac{1}{n} \sum_{j=1}^{n} G_j + \epsilon.$$
These three components are independent. Thus
$$\text{Var}(W) = V_W = \text{Var}\left(\frac{1}{m} \sum_{i=1}^{m}   Z_i \right) + \text{Var}\left(\frac{1}{n} \sum_{j=1}^{n} G_j \right) + \nu^2= \frac{s^2}{m} + \frac{t^2}{n} + \nu^2$$
Then the covariance between $Z_1$ and $W$ is given by
$$\text{Cov}(Z_1, W) = \text{Cov}\left( Z_1, \frac{1}{m}\sum_{i=1}^{m} Z_i\right) = \frac{1}{m} \text{Cov}\left( Z_1,  Z_1\right) = \frac{s^2}{m},$$
since the augmentation noise and $\epsilon$ are independent of $Z$.

For scalar jointly Gaussian $X$, $Y$ with $\text{Var}(X) = \sigma_X^2$, $\text{Var}(Y) = \sigma_Y^2$ and covariance $c$, the mutual information is
$$I(X, Y) = \frac{1}{2} \log \left( \frac{\sigma_X^2}{\sigma_X^2 - c^2/\sigma_Y^2} \right).$$
Apply this with $X = Z_1$ and $Y = W$ to get
$$I(Z_1, W) = \frac{1}{2} \log \left( \frac{s^2}{s^2 - \frac{(s^2/m)^2}{V_W}} \right) = \frac{1}{2} \log \left( \log \frac{m(s^2+t^2 /mn) + m^2\nu^2}{(m-1)s^2 + t^2/n + m^2\nu^2} \right)$$
Multiplying by $d$ yields and since all $i$ are exchangeable, we get
$$ \frac{1}{m} \sum_{i=1}^m\sqrt{2I_{\ell_\mathcal{G}}(Z_i;W)} = \sqrt{d \log \frac{m(s^2+t^2 /mn) + m^2\nu^2}{(m-1)s^2 + t^2/n + m^2\nu^2} }.$$
We next compute the mutual information $I_\ell^{Z_i}(G_j;W)$ between a single augmentation $G_j$ and $W$. Condition on $Z_i=z_i$ (treat $z_i$ as deterministic), we have
$$ W = \frac{1}{mn} \sum_{i=1}^{m} \sum_{j=1}^{n}  (Z_i + G_j) + \epsilon =  \frac{1}{n} G_j + \underbrace{ \frac{1}{n} \sum_{\ell=1, \ell\neq j}^{n} G_k  + \frac{1}{m} \sum_{k=1, , k\neq i}^{m} Z_i + \frac{1}{m} Z_i  +\epsilon}_{\eqdef N}.$$
Given $Z_i$ the term $\frac{1}{m} Z_i$ is deterministic and can be absorbed into a constant; $N$ is independent of $G_j$ (since the other $G_\ell$, $Z_k$ and $\epsilon$ are independent of $G_j$). Therefore, conditionally on $Z_i$, $W$ is again an additive Gaussian channel with input $X \eqdef \frac{1}{n} G_j$ having variance $\sigma_X^2 = t^2/n^2$ and independent noise $N$ having variance (per coordinate)
$$\text{Var}(N) = \frac{n-1}{n^2}t^2 + \frac{m-1}{m^2}s^2 + \nu^2.$$
Then,
$$ \text{Var}(W|Z_i) = \frac{t^2}{n^2} + \frac{n-1}{n^2}t^2 + \frac{m-1}{m^2}s^2 + \nu^2 =  \frac{t^2}{n} + \frac{m-1}{m^2}s^2 + \nu^2 $$
and covariance
$$ \text{Cov}(G_j,W|Z_i) = \text{Cov} \left(G_j, \frac{1}{n} \sum_{\ell = 1}^{n} G_\ell \right) = \frac{1}{n} \text{Var}(G_j) = \frac{t^2}{n} $$
Now plug these into the general formula cited above and multiply by $d$ we get
\begin{align*}
 I_\ell^{Z_i}(G_j;W) &= \frac{d}{2} \log \left( 1 + \frac{t^2/n^2}{\frac{n-1}{n^2}t^2 + \frac{m-1}{m^2}s^2 + \nu^2} \right)  \\
 &= \frac{d}{2} \log \left(1 + \frac{t^2/n^2}{(n-1)/n^2t^2 + (m-1)/m^2s^2 + \nu^2} \right)
\end{align*}
Plug into the theorem’s third term we get the desired result.

\end{proof}

\subsection{Proof of Proposition \ref{prop:bound_first_term} }
\label{proof:bound_first_term}
We recall the proposition
\begin{proposition*}[\textbf{KL divergence bound via group diameter}]
Let $\mathcal{D}$ be a distribution with a continuous density $p$ that is $L_p$-Lipschitz. Let $p_{\mathcal{G}}(z) \eqdef \underset{G \sim \mathcal{D}_\mathcal{G}}{\mathbb{E}}[p(Gz)]$ be the density of the distribution $\mathcal{D}_\mathcal{G} \circ \mathcal{D}$. Assume that $p_{\mathcal{G}}$ is uniformly lower bounded by  $p_{\mathcal{G}}\geq c >0 $. Then
$$D_{\mathrm{KL}} (\mathcal{D}\Vert  \mathcal{D}_\mathcal{G} \circ \mathcal{D}) \leq \frac{L_p \Delta_\mathcal{G}}{c}.$$
\end{proposition*}

\begin{proof}
If $p$ is the density of the distribution $\mathcal{D}$, then the density $p_{\mathcal{G}}$ of the distribution $\mathcal{D}_\mathcal{G} \circ \mathcal{D}$ is given by
$$p_{\mathcal{G}}(z) = \underset{G \sim \mathcal{D}_\mathcal{G}}{\mathbb{E}}[p(Gz)], \quad \text{for } z \in \mathcal{Z}.$$
Then, we have
$$D_{\mathrm{KL}} (\mathcal{D}\Vert  \mathcal{D}_\mathcal{G} \circ \mathcal{D}) = \int p(z) \log\frac{p(z)}{p_{\mathcal{G}}(z)}dz$$
Now we bound the difference between $p(z)$ and $p_{\mathcal{G}}(z)$ using the Lipschitz constant of $p$ and the maximum displacement under the group (i.e., the diameter). We have 
$$|p(z) - p_{\mathcal{G}}(z)| = | p(z)  - \underset{G \sim \mathcal{D}_\mathcal{G}}{\mathbb{E}} [p(Gz)]| \leq \underset{G \sim \mathcal{D}_\mathcal{G}}{\mathbb{E}} |p(z) - p(Gz) | \leq L_p \underset{G \sim \mathcal{D}_\mathcal{G}}{\mathbb{E}}d_\mathcal{Z}(z,Gz),$$
where the last inequality is given since $p$ is $L_p$-Lipschitz, i.e., $|p(z) - p(Gz) | \leq L_p d_\mathcal{Z}(z,Gz)$.

And since 
$$\underset{G \sim \mathcal{D}_\mathcal{G}}{\mathbb{E}}d_\mathcal{Z}(z,Gz) \leq \Delta_\mathcal{G},$$
we have
$$|p(z) - p_{\mathcal{G}}(z)| \leq L_p \Delta_\mathcal{G}.$$

Assume $p_\mathcal{G}(z)\geq c >0$ for all $z$, then we have
$$\log\frac{p(z)}{p_{\mathcal{G}}(z)} = \log \left(1 + \frac{p(z) - p_{\mathcal{G}}(z)}{p_{\mathcal{G}}(z)} \right) \leq \frac{p(z) - p_{\mathcal{G}}(z)}{p_{\mathcal{G}}(z)} \leq \frac{L_p \Delta_\mathcal{G}}{c}.$$
So, 
$$D_{\mathrm{KL}} (\mathcal{D}\Vert  \mathcal{D}_\mathcal{G} \circ \mathcal{D}) = \int p(z) \log\frac{p(z)}{p_{\mathcal{G}}(z)}dz \leq \int p(z) \frac{L_p \Delta_\mathcal{G}}{c} dz = \frac{L_p \Delta_\mathcal{G}}{c}.$$
\end{proof}

\subsection{Proof of Proposition \ref{prop:diff_mi} }
\label{proof:diff_mi}
We recall the proposition
\begin{proposition*}
Suppose that the original data distribution $\mathcal{D}$ is $\mathcal{G}$-invariant. Then, For every algorithm and group $\mathcal{G}$,
$$I_{\ell}(Z;W) - I_{\ell_\mathcal{G}}(Z;W) = \underset{Z \sim \mathcal{D} }{\mathbb{E}} \underset{G \sim \mathcal{D}_\mathcal{G} }{\mathbb{E}} [D_{KL}(P_{W \mid Z,G} || P^{\mathcal{G}}_{W \mid Z})] \geq 0.$$
Moreover, $I_{\ell}(Z;W) - I_{\ell_\mathcal{G}}(Z;W) = 0$ holds iff each $P_{W \mid Z,G}$ equals the averaged conditional almost surely.
\end{proposition*}
\begin{proof}
By Lemma \ref{lem:mi_kl_equivalence}, we have 
$$I_{\ell}(Z;W) =\underset{Z \sim \mathcal{D} }{\mathbb{E}} [D_{KL}(P_{W \mid Z} || P_{W })] \quad \text{and} \quad  I_{\ell_\mathcal{G}}(Z;W) = \underset{Z \sim \mathcal{D} }{\mathbb{E}} [D_{KL}( P^{\mathcal{G}}_{W \mid Z} || P_{W})].$$
Fix $z \in \mathcal{Z}$. For any prior $P_W$, we have
\begin{equation}
\label{eq:mixture_KL}
 \underset{G \sim \mathcal{D}_\mathcal{G} }{\mathbb{E}} [D_{KL}(P_{W \mid z,G} || P_{W })] - D_{KL}( P^{\mathcal{G}}_{W \mid z} ||P_W) = \underset{G \sim \mathcal{D}_\mathcal{G} }{\mathbb{E}} [D_{KL}(P_{W \mid Gz} || P^{\mathcal{G}}_{W \mid z})].
\end{equation}
Take expectation over $Z\sim\mathcal{D}$ yields
$$ \underset{Z \sim \mathcal{D} }{\mathbb{E}} \underset{G \sim \mathcal{D}_\mathcal{G} }{\mathbb{E}} [D_{KL}(P_{W \mid Z,G} || P_{W })] - \underset{Z \sim \mathcal{D} }{\mathbb{E}}[D_{KL}( P^{\mathcal{G}}_{W \mid Z} ||P_W)] = \underset{Z \sim \mathcal{D} }{\mathbb{E}} \underset{G \sim \mathcal{D}_\mathcal{G} }{\mathbb{E}} [D_{KL}(P_{W \mid Z,G} || P^{\mathcal{G}}_{W \mid Z})].$$
If the distribution $\mathcal{D}$ is $\mathcal{G}$-invariant, then
$$ \underset{Z \sim \mathcal{D} }{\mathbb{E}} \underset{G \sim \mathcal{D}_\mathcal{G} }{\mathbb{E}} [D_{KL}(P_{W \mid Z,G} || P_{W })] = \underset{Z \sim \mathcal{D} }{\mathbb{E}} [D_{KL}(P_{W \mid Z} || P_{W })] = I_{\ell}(Z;W).$$
Plugging these identifications into the previous equality yields
$$I_{\ell}(Z;W) - I_{\ell_\mathcal{G}}(Z;W) = \underset{Z \sim \mathcal{D} }{\mathbb{E}} \underset{G \sim \mathcal{D}_\mathcal{G} }{\mathbb{E}} [D_{KL}(P_{W \mid Z,G} || P^{\mathcal{G}}_{W \mid Z})].$$
The right-hand side of the above equation is an expectation of KL-divergences and therefore nonnegative, proving the inequality. Finally, $D_{KL}(P||Q) = 0$ iff $P=Q$ almost surely; therefore the difference is zero iff $P_{W \mid Z,G} = P^{\mathcal{G}}_{W \mid Z}$ almost surely. i.e., iff each conditional equals the averaged conditional almost surely.
\end{proof}

\subsection{Proof of Proposition \ref{prop:bound_third_term} }
\label{proof:bound_third_term}

We recall the proposition
\begin{proposition*}
Let $\mathcal{G}$ be a finite group acting on the input space $\mathcal{Z}$, and  let $P_{W|z,g}$ denote the conditional distribution of $W$ given an input $z \in \mathcal{Z}$ and a group transformation $g \in \mathcal{G}$ . Suppose that for all $z \in \mathcal{Z}$, the mapping $g \mapsto P_{W|z,g}$ is $C$-Lipschitz in total variation, i.e.,
$$\mathrm{TV}(P_{W|z,g}, P_{W|z,g'}) \leq  C d_\mathcal{Z}(gz,g'z), \forall g,g'\in \mathcal{G}.$$
Then, the augmentation mutual information term in Theorem~\ref{th:bu_like_bound_mi} satisfies
$$I_\ell^{z_i}(G_j;W) \leq \frac{C^2\Delta_\mathcal{G}^2}{\delta_z} , \forall i\in [m] \text{ and } j\in [n].$$
where $\delta_z \eqdef {\min}_{(g,w)} \mathcal{D}_\mathcal{G} (g) P_{W \mid z} (w) >0$ denotes the minimal joint probability under the product distribution $\mathcal{D}_\mathcal{G} \otimes P_{W \mid z}$.
\end{proposition*}

\begin{proof}
Let $z \in \mathcal{Z}$ be a fixed data point, and consider the random variables $G \sim \mathcal{D}_\mathcal{G}$ and $W \sim P_{W \mid z, G}$. Define the joint distribution $P_{G,W \mid z} \eqdef \mathcal{D}_\mathcal{G} \otimes P_{W \mid z, G}$ and let $P_{W \mid z} \eqdef \mathbb{E}_{G \sim \mathcal{D}_\mathcal{G}}[P_{W \mid z, G}]$ be the marginal.

To bound the mutual information $I_\ell^z(G; W)$, we use the reverse Pinsker’s inequality in Equation \ref{eq:reverse_pinsker}
$$I_\ell^{z}(G;W) = D_{KL}(P_{G, W \mid z}, \mathcal{D}_\mathcal{G} \otimes P_{W \mid z}) \leq \frac{1 }{\delta_z}  \mathrm{TV}(P_{G, W \mid z}, \mathcal{D}_\mathcal{G} \otimes P_{W \mid z})^2.$$
where $\delta_z \eqdef {\min}_{(g,w)} \mathcal{D}_\mathcal{G} (g) P_{W \mid z} (w) >0$.

Note that the joint distribution $P_{G, W \mid z}$ is defined by first sampling $G \sim \mathcal{D}_\mathcal{G}$ and then $W \sim P_{W \mid z, G}$, while the product of marginals $\mathcal{D}_\mathcal{G} \otimes P_{W \mid z}$ samples $G \sim \mathcal{D}_\mathcal{G}$ and $W \sim P_{W \mid z}$. Therefore,
$$\mathrm{TV}(P_{G, W \mid z}, \mathcal{D}_\mathcal{G} \otimes P_{W \mid z}) = \mathbb{E}_{G \sim \mathcal{D}_\mathcal{G}} \left[ \mathrm{TV}(P_{W \mid z, G}, P_{W \mid z}) \right].$$

Using Jensen’s inequality and the convexity of total variation
$$\mathrm{TV}(P_{W \mid z, G}, P_{W \mid z}) \leq \mathbb{E}_{G' \sim \mathcal{D}_\mathcal{G}} \left[ \mathrm{TV}(P_{W \mid z, G}, P_{W \mid z, G'}) \right].$$

Applying the smoothness assumption
$$\mathrm{TV}(P_{W \mid z, G}, P_{W \mid z, G'}) \leq C \cdot d_\mathcal{G}(Gz, G'z) \leq C \Delta_\mathcal{G}.$$

Therefore,
$$\mathrm{TV}(P_{G, W \mid z}, P_G \otimes P_{W \mid z}) \leq C \Delta_\mathcal{G}.$$

Then
$$I_\ell^z(G; W) \leq \frac{1 }{\delta_z} C^2 \Delta_\mathcal{G}^2.$$

\end{proof}

\clearpage
\section{Additional experimental details}     
\label{detail:experiments}
We evaluate the generalization bound proposed in Theorem \ref{th:bu_like_bound_mi} on the MNIST and FashionMNIST datasets, focusing on how data augmentation strength affects the bound components. These tasks are a standard image classification with 10 classes. We train a convolutional neural network (CNN) on the augmented datasets. The model consists of two convolutional layers (each followed by ReLU activation and max-pooling), followed by a fully connected layer and a final softmax classification layer. The CNN is trained using the Adam optimizer with a learning rate of $10^{-3}$, batch size of $128$, and for $20$ epochs. We ensure consistency across trials by fixing random seeds and averaging results over multiple runs (typically 5 seeds). The MINE model uses a two-layer multilayer perceptron (MLP) with ReLU activations, $128$ hidden units, and is trained with the Adam optimizer (learning rate $10^{-3}$) for $300$ gradient steps.

The geometric transformations, typically rotations and translations, are applied with increasing strength, allowing us to observe how the bounds vary with the severity of perturbation. For each training set, we generate $n=10$ augmented version per example to augment the dataset and approximate the orbit under the group of transformations $\mathcal{G}$.

\textbf{KL-divergence estimation.} To estimate the KL-divergence between the original data distribution $\mathcal{D}$ and the augmented distribution $\mathcal{D}_\mathcal{G} \circ \mathcal{D}$, we adopt a density ratio estimation approach using a discriminator-based classifier, following the standard density ratio trick \citep{choi2021featurized}. Specifically, we construct a binary classification task where samples from the original distribution $\mathcal{D}$ are labeled as class $0$, and samples from the augmented distribution $\mathcal{D}_\mathcal{G} \circ \mathcal{D}$ are labeled as class $1$. A neural network classifier is trained to distinguish between the two classes using cross-entropy loss. Once the discriminator is trained, we use its predicted probabilities to estimate the density ratio $r(x) = \frac{p_\mathcal{G}(x)}{p(x)}$. The KL divergence is then approximated using the empirical expectation
$$ D_{\mathrm{KL}}(\mathcal{D}\Vert \mathcal{D}_\mathcal{G} \circ \mathcal{D}) \approx \frac{1}{N} \sum_{i=1}^N \log \frac{1 - D(x_i)}{D(x_i)};$$
where $D(x_i)$ denotes the discriminator’s predicted probability that $x_i$ comes from the augmented distribution. This estimation is used as the first term in both generalization bounds in Theorem \ref{th:xu_like_bound_mi} and Theorem \ref{th:bu_like_bound_mi}, reflecting the degree of distributional shift induced by the data augmentation.

\textbf{Mutual information estimation.} To estimate the mutual information terms in the generalization bounds, we rely on a neural estimator based on Mutual Information Neural Estimation (MINE) \citep{belghazi2018mine}, which approximates mutual information via a lower bound based on the Donsker-Varadhan representation of KL-divergence. A shared encoder is used to reduce dimensionality before feeding inputs into the MINE network, which improves stability and reduces variance in estimation.


\end{document}